\documentclass{article}

\usepackage[final, main]{neurips_2026}
\usepackage[utf8]{inputenc} % allow utf-8 input
\usepackage[T1]{fontenc}    % use 8-bit T1 fonts
\usepackage{hyperref,comment,booktabs}       % hyperlinks
\usepackage{url,floatrow}            % simple URL typesetting
\usepackage{amsfonts}       % blackboard math symbols
\usepackage{nicefrac}       % compact symbols for 1/2, etc.
\usepackage{microtype}      % microtypography
\usepackage{color,bm}
\newfloatcommand{capbtabbox}{table}[][\FBwidth]
\usepackage{amsthm}
\usepackage{amssymb}
\usepackage{blindtext}
\usepackage{xcolor} 
\usepackage{mathrsfs}%huati font
\usepackage{amsmath}%split
\usepackage{graphicx} 
\usepackage{multirow}
\usepackage{tikz}
\usepackage{pgfplots}
\pgfplotsset{compat=1.17}
\usetikzlibrary{shapes.geometric, arrows.meta, positioning, shapes,fadings, decorations.pathmorphing, backgrounds, shadows,patterns,calc,3d,bending}

\newcommand{\Heven}{\mathcal{H}_{\text{even}}}
\newcommand{\Hodd}{\mathcal{H}_{\text{odd}}}
\usepackage{algpseudocode}
\usepackage{algorithm}
\newtheorem{theorem}{Theorem}

\newtheorem{lemma}{Lemma}

\newtheorem{definition}{Definition}

\newtheorem{remark}{Remark}

\newtheorem{axiom}{Axiom}

\title{The Urysohn Ladder: Recursive Metric Contraction for Scalable Continual Learning}
\author{%
  Xin~Li\thanks{This work was partially supported by NSF IIS-2401748 and BCS-2401398.
  The author acknowledges the use of large language model systems (Claude
  Opus 4.7 by Anthropic, ChatGPT 5 by OpenAI, and Gemini 2.5 Pro by Google) as
  assistive tools during the drafting process. All mathematical content has
  been reviewed and the author takes full responsibility for its correctness.} \\
  Department of Computer Science\\
  University at Albany\\
  Albany, NY 12222 \\
  \texttt{xli48@albany.edu} \\
  }

\begin{document}
\bibliographystyle{plain}
\maketitle

%``The correct reference frame to understand how the brain works is reference frame.'' - Jeff Hawkins
 
%\vspace{-0.2in}
\begin{abstract}
Continual learning systems face a fundamental geometric obstacle: as experience accumulates on a fixed-capacity manifold, covering numbers grow linearly with time, eventually forcing representational overlap and catastrophic interference. Prevailing approaches attack this problem by \emph{expansion} - projecting into higher-dimensional spaces via kernels, overparameterization, or replay. We argue the solution is the opposite: \emph{contraction}. We formalize abstraction as the \textbf{Urysohn Ladder}, a hierarchy of quotient maps that recursively collapse validated metric neighborhoods into compact tokens, converting unbounded ambient-space search into bounded navigation on a low-dimensional intrinsic scaffold. Geometrically, each collapsed token acts as a shortcut - a region of extreme metric contraction that bridges distant experiences, much like a wormhole in the representational manifold. We establish four results that collectively guarantee \emph{separability} (metric contraction renders nonlinearly entangled structure linearly separable at each quotient level, and this separability propagates faithfully through the entire hierarchy), \emph{bounded capacity} (covering numbers remain $O(1)$ per quotient level, independent of stream length), \emph{stability} (parity-partitioned flow/scaffold subspaces enable unbounded plasticity without catastrophic interference), and \emph{scalability} (inference cost scales with quotient distance, not ambient distance). We validate each claim empirically with pretrained models and real-world datasets. Moreover, we demonstrate the potential of Urysohn Ladder for scalable continual learning via scaffold amortization.
\end{abstract}

\section{Introduction}

Continual learning systems operate under a fundamental geometric tension: experience accumulates without bound, yet the representational capacity of any fixed-dimensional manifold is finite.
Classical results formalize this limitation sharply: in a fixed $d$-dimensional space, the number of stably separable regions is bounded by $O(d)$ \cite{cover1965geometry,vapnik1998statistical}, and the covering number of a trajectory of length $L$ grows as $\Theta(L/\epsilon)$ \cite{kolmogorov1959entropy,shalev2014understanding}.
As experience elongates, the trajectory eventually exceeds the covering capacity of the manifold, forcing representational overlap and catastrophic interference \cite{mccloskey1989catastrophic,french1999catastrophic}.
Prevailing approaches treat this as a \emph{packing problem}, attacking it by expansion: kernel methods project into higher-dimensional feature spaces \cite{shawe2004kernel}, overparameterized networks increase capacity through width \cite{allen2019learning}, and replay methods expand the effective support of the training distribution \cite{parisi2019continual}. These strategies succeed in finite regimes but are governed by an unsustainable scaling law - representational resources must grow at least linearly with experience. In open-ended environments, the curse of dimensionality ensures that the volume required to separate entangled manifolds grows exponentially, eventually overwhelming any finite expansion \cite{bellman1957dynamic}. 
We argue the solution is the \emph{opposite} of expansion. Rather than increasing representational volume, we actively \emph{contract} the metric over validated submanifolds, collapsing them into compact tokens via quotient maps. The contraction view inverts the standard paradigm: instead of adding dimensions to untangle complex sets, we shrink the metric until the sets become topologically disjoint. A classical result, Urysohn's Lemma \cite{willard2012general}, guarantees that a stable separating function exists regardless of the ambient dimension (Fig.~\ref{fig:CoD}). %We call the resulting hierarchy of quotient maps the \textbf{Urysohn Ladder}.
 
%Biologically, this perspective aligns with the observation that cortical systems do not expand capacity to accommodate new experience; instead, they compress temporal structure through mechanisms such as hippocampal replay and attractor formation, effectively contracting functional distances between related states \cite{buzsaki2006rhythms}. Rather than expanding the container to fit more points, they reshape the container to separate the contents.
\begin{figure}
\centering
\resizebox{.75\linewidth}{!}{%
    \begin{tikzpicture}[font=\sffamily, >=LaTeX]
 
    % --- DEFINITIONS ---
    \definecolor{classA}{RGB}{65, 105, 225} % Royal Blue
    \definecolor{classB}{RGB}{220, 20, 60}  % Crimson
 
    % --- PANEL 1: THE GEOMETRIC CURSE (Entanglement) ---
    \node (p1) at (-1,0) {};
    
    % Draw intertwined manifolds (blobs)
    % Class A (Blue) - Spreading out, "filling volume"
    \filldraw[classA!20, draw=classA, thick, fill opacity=0.6] 
        (-2,-1.5) to[out=90,in=180] (-1,0.5) to[out=0,in=180] (1, -0.5) 
        to[out=0,in=270] (2,1) to[out=90,in=0] (0, 2) 
        to[out=180,in=90] (-2.5, 0.5) to[out=270,in=180] (-2,-1.5);
    
    % Class B (Red) - Interlocking with A
    \filldraw[classB!20, draw=classB, thick, fill opacity=0.6] 
        (-0.5,-2) to[out=0,in=270] (1.5,-0.5) to[out=90,in=0] (0.5, 1.5) 
        to[out=180,in=0] (-1.5, -0.5) to[out=180,in=90] (-2.5, -1.5) 
        to[out=270,in=180] (-1, -2.5) to[out=0,in=270] (-0.5,-2);
 
    % "Touching" Point (The Singularity)
    \node[circle, fill=black, inner sep=1.5pt] (contact) at (-0.2, -0.2) {};
    \node[above right, scale=0.7] at (contact) {Contact $\delta \to 0$};
 
    % Annotations
    \node[above=2.2cm of p1, align=center] (title1) {\textbf{1. The Geometric Curse}\\(Entangled in High-D)};
    \node[below=2.5cm of p1, align=left, scale=0.75] (desc1) {
        \textbf{Problem:} Manifolds weave.\\
        Separation Energy $\propto 1/\delta \to \infty$.\\
        Needs exponential data to separate
    };
 
    % --- ARROW: THE CONDENSATION OPERATOR ---
    \draw[->, line width=1mm, color=black!70] (2.5, 0) -- (6.5, 0) 
        node[midway, above, scale=0.8] {Metric Collapse}
        node[midway, below, scale=0.8] {$\Psi_\rho$ (Quotient Map)};
 
    % --- PANEL 2: THE TOPOLOGICAL SOLUTION (Urysohn) ---
    \node (p2) at (8,0) {};
 
    % Draw Collapsed/Contracted Manifolds
    % Class A (Blue) - Shrunk to centroid
    \filldraw[classA!80, draw=classA, thick] 
        (6.5, 1.5) circle (0.4cm) node[white, font=\bfseries] {$A$};
    
    % Class B (Red) - Shrunk to centroid
    \filldraw[classB!80, draw=classB, thick] 
        (9.5, -1.5) circle (0.4cm) node[white, font=\bfseries] {$B$};
 
    % The Margin
    \draw[<->, dashed, thick] (6.9, 1.2) -- (9.1, -1.2) node[midway, above right, font=\small] {Margin $\delta > 0$};
 
    % The Urysohn Function (Gradient Bridge)
    \shade[left color=classA!50, right color=classB!50, opacity=0.5] 
        (6.5, 1.5) -- (9.5, -1.5) -- (10.5, -1.5) -- (7.5, 1.5) -- cycle;
    \node[rotate=-45, scale=0.7] at (8, -0.5) {Urysohn Function $f$};
 
    % Annotations
    \node[above=2.2cm of p2, align=center] (title2) {\textbf{2. The Topological Blessing}\\(Disjoint Sets)};
    \node[below=2.5cm of p2, align=center, scale=0.75] (desc2) {
        \textbf{Insight:} Collapse creates margin.\\
        Separation Energy $\propto 1/\delta < \infty$.\\
        Function $f$ exists regardless of Dimension $D$.
    };
 
    % --- BOTTOM: ENERGY PROFILE ---
    \begin{scope}[shift={(2, -5.)}]
        \draw[->] (0,0) -- (5,0) node[right] {Connectivity (Margin $\delta$)};
        \draw[->] (0,0) -- (0,2.6) node[above] {Energy Cost (Lipschitz)};
        
        % Curve
        \draw[domain=0.2:4.5, smooth, variable=\x, thick, orange] plot ({\x}, {0.5/\x});
        
        % Points
        \fill[red] (0.2, 2.5) circle (2pt) node[right] {Curse (Entangled)};
        \fill[green!60!black] (3.0, 0.16) circle (2pt) node[above right] {Urysohn (Collapsed)};
    \end{scope}
    \end{tikzpicture}
    }
    \vspace{0.1in}
\caption{Turing Dimensionality from the curse of Lipschitz Continuity to the blessing of Connectivity via metric collapse.
Left (Geometric Curse): In the ambient high-dimensional space, class manifolds (blue and red) are entangled, effectively touching at the decision boundary (margin $\delta \to 0$). Separating these sets requires a function with infinite slope, leading to the metabolic singularity where learning diverges.
Right (Topological Blessing): Instead of estimating density in the volume, the system applies a metric collapse operator ($\Psi_\rho$), contracting the manifolds toward their centroids (quotient map), which creates a synthetic margin $\delta > 0$.
The bottom plot illustrates the thermodynamic phase transition: the system moves from a high-energy entangled state to a low-energy separated state.}
\vspace{-0.2in} 
\label{fig:CoD}
\end{figure}
 
The core mechanism is \emph{recursive metric contraction}: validated sub-trajectories on the temporal manifold are collapsed via quotient maps, producing a hierarchy $\mathcal{M}_0 \xrightarrow{q_0} \mathcal{M}_1 \xrightarrow{q_1} \cdots \xrightarrow{q_D} \mathcal{M}_D$ of progressively simpler spaces. Geometrically, each collapsed token acts as a wormhole - a region of extreme curvature that brings distant events into immediate proximity. Inference then proceeds by navigating the contracted compact scaffold rather than scanning the original trajectory. We reframe the curse of dimensionality as a \emph{Curse of Lipschitz Continuity} \cite{bach2017breaking}: in the raw input stream, entangled manifolds require infinite energy to separate via continuous functions \cite{vonLuxburgBousquet2004}; quotient contraction creates the margin needed for finite-energy separation.
We make our geometric intuition precise through four formal results:
1) \textbf{Separability} (Theorems~\ref{thm:topological_collapse} and~\ref{thm:urysohn_quotient}): metric contraction renders nonlinearly entangled structure linearly separable via quotient topology, and separability propagates faithfully through the entire quotient hierarchy. 
2) \textbf{Stability} (Theorem~\ref{thm:parity_stability}): parity-partitioned flow/scaffold subspaces enable unbounded plasticity without catastrophic interference \cite{mccloskey1989catastrophic}.
3) \textbf{Bounded Capacity} (Theorem~\ref{thm:bounded_capacity}): covering numbers remain $O(1)$ per quotient level, independent of stream length $L$.
4) \textbf{Scalability} (Theorem~\ref{thm:urysohn_quotient}): inference cost scales with quotient distance, not ambient distance \cite{gromov1999metric}.
Sections~\ref{sec:geometric_framework}-\ref{sec:recursive_depth} develop the theoretical framework. Section~\ref{sec:empirical} validates each theorem empirically on real-world data (e.g., MNIST, CIFAR-10) and pretrained models (e.g., GPT-2, ResNet-18). Additional experimental results on the scalability of recursive metric contraction are provided in the Appendix.

\section{Geometric Framework}
\label{sec:geometric_framework}

\noindent\textbf{Setup.}
We model experience as a trajectory $\gamma(t)$ on a Riemannian manifold $(\mathcal{M}, g)$, where $g$ denotes the metric on $\mathcal{M}$ \cite{gromov1999metric}.
In our view, learning reshapes the geometry of the space in which inference occurs, shortening geodesic distances between causally related regions \cite{bengio2013representation}.
The covering number of a flat trajectory grows as $N(\epsilon, \mathcal{M}) \propto L$, leading to catastrophic interference when it exceeds hardware capacity \cite{mccloskey1989catastrophic}.
The key question is: \emph{how should the geometry of $\mathcal{M}$ be transformed to prevent this?}
Our approach is inspired by a classical result known as Urysohn's Lemma in point-set topology \cite{willard2012general}:
 
\begin{lemma}[Urysohn's Lemma]
\label{lem:urysohn}
Let $X$ be a normal topological space. Let $A$ and $B$ be two disjoint closed subsets of $X$.
Then there exists a continuous function $f: X \to [0, 1]$ such that:
$f(x) = 0 \forall x \in A ~\text{and}~ f(x) = 1 \forall x \in B$.
The function $f$ is often referred to as a \textit{separating function}.
\end{lemma}
 
\begin{remark}[From topological separability to constructive abstraction]
Lemma~\ref{lem:urysohn} provides an existence guarantee: once two hypotheses occupy disjoint closed
supports, there always exists a continuous \emph{separating function} $f$ that can act as a stable gate
between them. %However, the lemma is intentionally non-constructive: it does not specify \emph{how} a learning system should transform entangled experience into disjoint closed sets, nor how such separation can be made robust and reusable under continual accumulation.  
Our constructive answer has two parts:
1) \textbf{metric contraction via quotienting}: Given a validated subset $U \subset \mathcal{M}$, collapse it to a single point via a quotient map, contracting internal distances to zero. Metric contraction simultaneously changes the topology ($\beta_1 \to \beta_0$) and compresses covering numbers, yielding separability.
2) \textbf{parity alternation as a separator pump}: An expansion stroke (odd phase) explores the representation and identifies valid candidates; a contraction stroke (even phase) applies contraction to those candidates. The orthogonality between phases ensures that each contraction increases the margin $\delta$ without disturbing previously stored tokens, yielding stability.
\end{remark}

\subsection{Metric Contraction via Topological Quotienting}
 
\noindent\textbf{Metric Contraction via Quotienting}
In our geometric framework, the gap left by the non-constructive Urysohn's Lemma is closed by \emph{quotienting} \cite{burago2001course}: an explicit operation that \emph{creates} separated supports by contracting validated recurrent structure into single tokens. Rather than
treating $f$ as the primitive, we treat the formation of the closed sets to be separated as the
primitive and implement it geometrically via a controlled metric contraction on selected
subsets of the manifold (e.g., Equiangular Tight Frame \cite{papyan2020prevalence}).
%We formalize abstraction through a geometric operation that actively contracts the metric over selected subsets of the manifold.
 
\begin{definition}[Metric Contraction Operator]
\label{def:metric_contraction}
Let $U \subset \mathcal{M}$ denote a subset of the temporal manifold corresponding to a validated structure (e.g., a recurrent or causally closed trajectory segment).
A \textbf{condensation operator} $\Psi$ induces a modified metric $\Psi(g)$ such that, for all $x,y \in U$,
$d_{\Psi(g)}(x,y) \to 0$.
Topologically, $\Psi$ defines a \textbf{quotient map} or a \textbf{condensation operator}
$q : \mathcal{M} \rightarrow \mathcal{M}/\!\sim$,
where all points in $U$ are identified with a single equivalence class $p^\ast$.
We refer to $p^\ast$ as a \emph{token}.
\end{definition}
 
\noindent\textbf{Toy Example.}
Consider a trajectory $\gamma$ of $L=6$ token-level hidden states encoding the phrase ``the cat sat on the mat.'' Before condensation, each token occupies a distinct point on $\mathcal{M}_0$; the geodesic distance from the first to the last is $d_g \propto L$, and covering the trajectory requires $N(\epsilon) = O(L/\epsilon)$ balls. Now suppose the system has validated this trajectory as a self-consistent recurrent structure (a closed cycle $\beta_1$). The condensation operator $\Psi$ collapses all six points into a single equivalence class $p^\ast$ (a token) in the quotient $\mathcal{M}_1 = \mathcal{M}_0/{\sim}$. In $\mathcal{M}_1$, the internal diameter of that region is exactly zero and the covering number is 1, while distances to \emph{other} tokens are preserved. %This toy example is metric contraction in the most literal sense: the quotient map $q$ contracts a region of diameter $O(L)$ to a point, and the topological transition $\beta_1 \to \beta_0$ is the geometric mechanism by which the covering number decreases.
 
\noindent\textbf{The Dual View: Topology and Metric as One Operation.}
The condensation operator $\Psi$ admits two complementary descriptions that are formally equivalent. In the \emph{topological} view, a validated recurrent trajectory, a closed 1-cycle $\beta_1$ in the representation, is collapsed into a connected component $\beta_0$: homological dimension decreases by one. In the \emph{metric} view, the internal geodesic diameter of the trajectory collapses from $O(L)$ to zero, and the covering number for that region drops from $O(L/\epsilon)$ to $1$. These are the same operation seen through different lenses: the quotient map $q$ simultaneously kills a cycle (changing topology) and contracts the metric (shrinking diameter). %This is why the framework can derive both separability (Theorem~\ref{thm:topological_collapse}, a topological guarantee) and bounded capacity (Theorem~\ref{thm:bounded_capacity}, a metric guarantee) from the same construction. The parity inversion $\beta_1 \to \beta_0$ is precisely the $\rho$-compression of the covering number.
 
\noindent\textbf{Wormhole Interpretation.}
Under the original metric $g$, the endpoints $x_{t_{\text{start}}}$ and $x_{t_{\text{end}}}$ of a long temporal sequence may be separated by a large geodesic distance.
Metric contraction introduces a region of extreme curvature that effectively pinches the manifold, bringing these distant points into immediate proximity \cite{gromov1999metric,memoli2007gromov}.
From the perspective of inference, metric contraction creates a \emph{topological shortcut}: a path of negligible length that bypasses the original temporal extent.
We use the term ``wormhole'' as an intuitive description of such a shortcut effect \cite{chung1997spectral,coifman2006diffusion}, emphasizing that inference cost is reduced by deforming the metric rather than by traversing the original trajectory.
 
\noindent\textbf{Hierarchical Quotient Spaces.}
Condensation can be applied recursively, producing a hierarchy of progressively simplified manifolds:
$\mathcal{M}_0 \xrightarrow{q_0} \mathcal{M}_1 \xrightarrow{q_1} \dots \xrightarrow{q_D} \mathcal{M}_D$.
At each level $k$, the quotient manifold $\mathcal{M}_k$ has strictly smaller effective diameter and covering number than $\mathcal{M}_{k-1}$.
A token at level $k$ represents a potentially long temporal segment of $\mathcal{M}_0$, so an unbounded history can be represented by a bounded number of tokens at sufficiently high abstraction.
While $\mathcal{M}_0$ may have diverging volume, the upper-level quotients remain compact \cite{burago2001course}, and inference operates by traversing these compact spaces, which we call ``Urysohn Ladder'' (Fig. \ref{fig:log_scaling}).
 
\paragraph{From the Curse to the Blessing of Dimensionality.}
Recursive quotienting resolves the Curse of Dimensionality and the plasticity-stability dilemma \cite{de2021continual} by attacking their shared root cause: \emph{overlapping supports} in representation space. In the entangled regime, distance-based retrieval is unstable and gradient updates overwrite previously acquired structure \cite{bengio2013representation}. The quotient operator $\Psi$ converts overlapping supports into disjoint closed sets, enabling stable gating. Geometrically, it replaces ambient search with navigation on an intrinsic scaffold; dynamically, it localizes learning to the fiber of the active token, so updates refine internal structure without perturbing other tokens. The same mechanism that lowers \emph{inference cost} (from search to navigation) simultaneously lowers \emph{learning cost} (from global rewrites to localized updates), turning high ambient dimension from a liability into an asset for routing \cite{mccloskey1989catastrophic}.

\subsection{The Parity Alternation Principle}
 
Separating complex, entangled manifolds requires more than a single geometric projection - it requires an iterative construction. We formalize our intuition as the \textbf{parity alternation principle}: the state space decomposes into two complementary subspaces with distinct computational roles \cite{mcclelland_why_1995}. An \emph{odd-parity flow} ($\Hodd$) supports plastic search and exploratory neighborhood selection, while an \emph{even-parity scaffold} ($\Heven$) validates and consolidates structure into stable tokens. The connection to the constructive proof of Urysohn's Lemma \cite{urysohn_zum_1925} and the reason why alternation is structurally necessary, is developed below after the formal definitions.
 
\begin{definition}[Parity Partitioning]
The \textbf{parity alternation principle} posits that the cognitive state space $\mathcal{M}$ is a \emph{homologically partitioned complex} admitting an (approximate) parity decomposition
$\mathcal{M} \approx \Hodd \oplus \Heven$,
where $\Hodd$ denotes the odd-parity subspace supporting transient dynamical features and $\Heven$ denotes the even-parity subspace supporting consolidated structural features.
\end{definition}
 
\begin{axiom}[Functional Conjugacy]
The parity alternation principle asserts a functional dichotomy between two complementary computational modes, grounded in the two strokes of the Urysohn separator construction:
1) \textbf{Flow ($\Hodd$).} A high-entropy \emph{expansion} phase in which the system explores candidate neighborhoods, traces trajectories, and identifies boundary failures;
2) \textbf{Scaffold ($\Heven$).} A low-entropy \emph{contraction} phase in which validated recurrent structure is stabilized as invariant units.
\end{axiom}

\begin{remark}
We adopt the notation $\Hodd$ and $\Heven$ as shorthand reflecting the dominant homological instance ($\beta_1 \to \beta_0$), while noting that the functional dichotomy is derived from the Urysohn construction rather than from general properties of even- and odd-dimensional homology groups.
The canonical topological substrate of $\Hodd$ is $\beta_1$: closed loops that encode temporal recurrence, causal closure, and exploratory composition. Operationally, the phase aligns with \emph{active inference} \cite{friston2017active} and corresponds to the open-set proposal step ($U_r$) in the constructive Urysohn proof.  The canonical topological substrate of $\Heven$ is $\beta_0$: connected components that encode discrete, reusable tokens. Operationally, the alternating phase aligns with \emph{memory consolidation} \cite{squire2015memory} and corresponds to the closure-enforcement step ($\overline{U_r} \subset U_s$) in the constructive Urysohn proof.
\end{remark}
 
\begin{axiom}[Parity-Inverting Condensation Operator]
Learning is defined not as continuous deformation within a fixed space, but as a \textbf{parity-inverting map}
$\Psi:\Hodd(\mathcal{M}) \rightarrow \Heven(\mathcal{M})$,
which performs \textbf{topological condensation}: a validated odd-parity cycle is collapsed into an even-parity token. Equivalently, $\Psi$ implements a quotienting step that converts a \emph{process} into a \emph{thing}.
\end{axiom}

\begin{remark}
Each expansion-contraction cycle increases the margin $\delta$ between competing class supports. Before the cycle, interleaved class manifolds may have a margin $\delta \approx 0$. The odd phase identifies \emph{where} they interleave (boundary failures). The even phase collapses the validated interior of each class, pulling its support away from the boundary. After $k$ cycles, $\delta$ grows monotonically. Eventually, the class supports become topologically disjoint and at that point Urysohn's Lemma activates: the separating function $f$ is guaranteed to exist with finite Lipschitz constant $L \propto 1/\delta$. Parity alternation does not invoke Urysohn's Lemma; it \emph{manufactures the precondition} that makes the lemma applicable. Neither stroke alone suffices: expansion without contraction yields exploration that never consolidates; contraction without expansion yields premature closure over bad structure. The alternation is the minimal mechanism that constructively achieves topological disjointness.
\end{remark} 

\noindent\textbf{Mechanism: Parity alternation as a separator pump.}
The connection between parity alternation and Urysohn's Lemma is not analogical but structural \cite{gentner_structure-mapping_1983}. The constructive proof of Urysohn's Lemma builds the separating function $f$ by an iterative procedure: for each dyadic rational $r$, find an open set $U_r$ (expansion) such that the closures nest: $\overline{U_r} \subset U_s$ whenever $r < s$ (contraction). The alternation between proposing open neighborhoods and enforcing closure containment is not a stylistic choice in the proof but a structural necessity.
Parity alternation maps directly onto the above construction. The \emph{odd phase} (flow/$\Hodd$) is the expansion stroke: the system explores, proposes candidate neighborhoods, and measures where the current representation fails - points that should be separated but are not yet. In neural terms, odd phase is the forward pass exposing high-loss boundary failures, active inference probing new trajectories \cite{friston2017active}. The \emph{even phase} (scaffold/$\Heven$) is the contraction stroke: the system takes closures, validates proposed neighborhoods, and consolidates them into compact tokens by contracting the metric within each validated region. In neural terms, even phase is offline consolidation, collapsing a trajectory into a reusable chunk \cite{squire2015memory}.

\noindent\textbf{Implication: resolving stability-plasticity by transversalization.}
In standard neural networks, plasticity and stability compete for the same synaptic substrate \cite{mccloskey1989catastrophic}. Under parity alternation, they are architecturally decoupled: plasticity occurs in $\Hodd$ via exploratory deformation of flows, while stability is maintained in $\Heven$ as invariant token structure that gates retrieval. If condensation drives learning directions to be transverse to the existing scaffold, interference vanishes:
$\langle \nabla \Psi(\gamma_{\text{new}}), \nabla S_{\text{old}} \rangle \rightarrow 0$.
The system can therefore sustain open-ended plasticity in the odd phase while preserving token-level invariants in the even phase (formalized in Theorem~\ref{thm:parity_stability}).

\section{Main Results}
 
The three results below are consequences of the single expansion-contraction mechanism developed in Section~\ref{sec:geometric_framework}, each isolating a different guarantee. Parity alternation constructs disjointness (enabling separability), metric contraction bounds covering numbers (enabling bounded capacity), and the orthogonal phase structure prevents interference (enabling stability). Each theorem isolates one face of the expansion-contraction mechanism to make the guarantee precise and falsifiable.

%\noindent\textbf{Preliminaries and Definitions.} Let $(\mathcal{M}_0, g_0)$ denote the initial temporal manifold induced by the input stream of length $L$, and let $N(\epsilon, \mathcal{M})$ denote the \emph{covering number} of a metric space $(\mathcal{M}, g)$, the minimal number of $\epsilon$-radius balls required to cover $\mathcal{M}$ \cite{kolmogorov1959entropy}. We interpret the representational demand of a learning system as proportional to this covering number.

\noindent\textbf{The Topological Collapse Separability Theorem}
The standard framing of high-dimensional learning treats separability as a geometric curse requiring exponential data \cite{bellman1957dynamic}. We argue the obstacle is not volume but \emph{connectivity}: Urysohn's Lemma guarantees that separability requires only a continuous deformation that renders supports disjoint, not an increase in dimension.
 
\begin{theorem}[Quotient Collapse Preserves Separability]
\label{thm:topological_collapse}
Let $\mathcal{M}$ be a normal topological space and let $A,B\subset \mathcal{M}$ be disjoint closed sets.
By Urysohn’s lemma, there exists a continuous $f:\mathcal{M}\to[0,1]$ such that $f(A)=\{0\}$ and $f(B)=\{1\}$.
Define an equivalence relation $x\sim_f y \iff f(x)=f(y)$ and let $q:\mathcal{M}\to\tilde{\mathcal{M}}:=\mathcal{M}/\!\sim_f$ be the quotient map.
Then:
1) The images $q(A)$ and $q(B)$ are distinct singleton points in $\tilde{\mathcal{M}}$ (the equivalence classes at levels $0$ and $1$).
2) There exists a unique continuous $\bar f:\tilde{\mathcal{M}}\to[0,1]$ such that $f=\bar f\circ q$, and $\bar f(q(A))=0$, $\bar f(q(B))=1$.
3) Consequently, $q(A)$ and $q(B)$ are separable by the threshold rule $\mathbf{1}[\bar f(\cdot) > 1/2]$.
\end{theorem}
 
Theorem~\ref{thm:topological_collapse} establishes that metric contraction fundamentally alters the separability regime identified by Cover \cite{cover1965geometry}: collapsing validated submanifolds into quotient points deforms $\mathcal{M}_0$ into a compact hierarchy $\{\mathcal{M}_k\}$ on which inference proceeds via short geodesic traversals rather than ambient-space search. Paths that stay within a single class region become shorter. Paths that cross boundaries become relatively longer. The geodesics are deformed so they preferentially follow the grain of the classification, threading along class regions rather than cutting across them.

\noindent\textbf{The Parity-Partitioned Stability Theorem}
While Theorem~\ref{thm:topological_collapse} guarantees that a single manifold can be collapsed to achieve linear separability, it does not prevent the metric distortion from interfering with prior stored memories \cite{french1999catastrophic}. The Parity Alternation Principle (Section~\ref{sec:geometric_framework}) addresses the interference by decomposing the state space into orthogonal subspaces with segregated update rules. We now formalize the decomposition and its stability guarantee.
 
\begin{definition}[Orthogonal Parity Decomposition]
Let the cognitive state space be a Riemannian product manifold $\mathcal{M} = \mathcal{M}_{\text{fast}} \times \mathcal{M}_{\text{slow}}$ with metric $g = g_{\text{fast}} \oplus g_{\text{slow}}$.
Let $\theta = [\theta_F, \theta_S]$ denote the parameter vectors governing the geometry of each subspace.
\end{definition}
 
The product structure ensures that gradient updates in one subspace have zero projection onto the other. The following theorem formalizes the stability guarantee.
 
\begin{theorem}[Parity-Partitioned Stability]
\label{thm:parity_stability}
Let parameters decompose as $\theta=(\theta_F,\theta_S)\in\Theta_F\times\Theta_S$,
and suppose the system alternates two update phases:
$\text{(Flow)}:\ \Delta\theta_S=0,
\text{(Scaffold)}:\ \Delta\theta_F=0$.
Assume the metric on parameter space is block-diagonal,
$g = g_F \oplus g_S$, so that the induced inner product satisfies
$\langle (u_F,0),(0,v_S)\rangle_g = 0$.
Then the cross-interference term between phases vanishes:
$\langle \Delta\theta^{(F)}, \Delta\theta^{(S)}\rangle_g = 0$.
In particular, if a memory functional $R(\theta_S)$ depends only on $\theta_S$,
then Flow-phase updates leave $R$ invariant: $R(\theta_F+\Delta\theta_F,\theta_S)=R(\theta)$.
\end{theorem}
 
Theorem~\ref{thm:parity_stability} provides the formal guarantee that parity-segregated updates eliminate cross-task interference. In standard networks, plasticity and stability compete within a shared parameter space; parity alternation decouples them architecturally. Because metric collapse operates entirely within the high-entropy flow space ($\Hodd$), the low-entropy scaffold ($\Heven$) remains metrically invariant and updates refine task-specific flow structure without perturbing stored tokens. We empirically validate zero-interference guarantee in Experiment~3 (Section~\ref{sec:empirical}).
 
\begin{figure}[h]
\centering
\resizebox{\linewidth}{!}{%
\begin{tikzpicture}[
    scale=0.9, transform shape,
    manifold/.style={
        draw=gray!50, 
        fill=gray!10, 
        rounded corners, 
        minimum height=1.2cm, 
        inner sep=5pt,
        align=center
    },
    traj/.style={
        thick, 
        decorate, 
        decoration={snake, amplitude=1pt, segment length=5pt},
        color=blue!70
    },
    condensed_traj/.style={
        very thick,
        color=blue!90!black,
        smooth
    },
    raw_point/.style={circle, fill=blue!50, inner sep=1pt},
    condensed_point/.style={diamond, fill=red!70, inner sep=2pt},
    map_arrow/.style={->, -{Stealth[scale=1.2]}, thick, color=red!80!black},
    contraction_cone/.style={
        fill=red!30, 
        opacity=0.4, 
        draw=none
    },
    dim_line/.style={
        |<->|, 
        thick, 
        >=Stealth,
        color=black!70
    },
    % Parity phase styles
    odd_phase/.style={
        fill=orange!20,
        draw=orange!60!black,
        rounded corners=3pt,
        thick,
        minimum width=2.8cm,
        minimum height=0.9cm,
        align=center,
        font=\footnotesize
    },
    even_phase/.style={
        fill=blue!15,
        draw=blue!60!black,
        rounded corners=3pt,
        thick,
        minimum width=2.8cm,
        minimum height=0.9cm,
        align=center,
        font=\footnotesize
    },
    pump_arrow/.style={
        ->, -{Stealth[scale=1]},
        very thick
    }
]
 
% ============================================================
% LEFT SIDE: THE URYSOHN LADDER (quotient hierarchy)
% ============================================================
 
% --- LEVEL 0: Raw Input ---
\node[manifold, minimum width=12cm] (M0) at (0,0) {};
\node[below right, gray] at (M0.north west) {$\mathcal{M}_0$ (Base Manifold)};
% Raw trajectory with entangled classes
\draw[thick, red!70, smooth] (-5.2, 0.15) to[out=20,in=160] (-3, -0.1) to[out=-20,in=200] (-1, 0.2) to[out=20,in=160] (1, -0.15) to[out=-20,in=200] (3, 0.1) to[out=20,in=160] (5.2, -0.1);
\draw[thick, blue!70, smooth] (-5.2, -0.15) to[out=-20,in=200] (-3, 0.1) to[out=20,in=160] (-1, -0.2) to[out=-20,in=200] (1, 0.15) to[out=20,in=160] (3, -0.1) to[out=-20,in=200] (5.2, 0.1);
\node at (0, -0.45) {\footnotesize Entangled class manifolds ($\delta \approx 0$)};
% Dimension L
\draw[dim_line] (-5.5, -0.9) -- (5.5, -0.9) node[midway, below] {Linear Diameter $\propto L$};
 
% Margin indicator L0 (near zero)
\draw[<->, orange!80!black, thick] (5.8, -0.15) -- (5.8, 0.15);
\node[right, orange!80!black, font=\scriptsize] at (5.9, 0) {$\delta_0 \!\approx\! 0$};

% --- LEVEL 1: After first expansion-contraction cycle ---
\node[manifold, minimum width=7cm] (M1) at (0, 3.5) {};
\node[below right, gray] at (M1.north west) {$\mathcal{M}_1$ (Quotient Manifold)};
% Condensed points L1 - now separated
\coordinate (c1_1) at (-2.5, 3.7);
\coordinate (c1_2) at (0, 3.3);
\coordinate (c1_3) at (2.5, 3.7);
% Class A tokens (red)
\fill[red!70] (-2.5, 3.7) circle (3pt);
\fill[red!70] (2.5, 3.7) circle (3pt);
% Class B tokens (blue)
\fill[blue!70] (0, 3.3) circle (3pt);
% Connections
\draw[condensed_traj, red!50] (-2.5, 3.7) to[out=-10, in=190] (2.5, 3.7);
\draw[condensed_traj, blue!50] (-1.2, 3.3) -- (1.2, 3.3);
 
% Margin indicator L1 (growing)
\draw[<->, green!60!black, thick] (3.1, 3.3) -- (3.1, 3.7);
\node[right, green!60!black, font=\scriptsize] at (3.2, 3.5) {$\delta_1 > 0$};

% --- Contraction Cones q0 ---
\begin{pgfonlayer}{background}
    \fill[contraction_cone] (-5.5, 0.3) -- (-1.5, 0.3) -- (c1_1) -- cycle;
    \fill[blue!15, opacity=0.4, draw=none] (-1.5, 0.3) -- (1.5, 0.3) -- (c1_2) -- cycle;
    \fill[contraction_cone] (1.5, 0.3) -- (5.5, 0.3) -- (c1_3) -- cycle;
\end{pgfonlayer}

% --- DOTS ---
\node at (0, 5.2) {\Huge $\vdots$};
\node[right, gray, font=\footnotesize, align=left] at (1, 5.2) {Iterated\\pump cycles};

% --- LEVEL D: Final Manifold ---
\node[manifold, minimum width=3cm] (MD) at (0, 7) {};
\node[below right, gray] at (MD.north west) {$\mathcal{M}_D$ (Top Manifold)};
% Two well-separated tokens
\fill[red!90] (-0.8, 6.7) circle (4pt);
\fill[blue!90] (0.8, 6.7) circle (4pt);
\node[red!80!black, font=\scriptsize, above] at (-0.8, 6.75) {$A$};
\node[blue!80!black, font=\scriptsize, above] at (0.8, 6.75) {$B$};
% Separator
\draw[green!60!black, very thick, dashed] (0, 6.5) -- (0, 7.5);
\node[green!60!black, font=\scriptsize, below] at (0, 6.45) {$f$};
 
% Dimension Bounded
\draw[dim_line] (-1.5, 7.8) -- (1.5, 7.8) node[midway, above, align=center] {Bounded Diameter\\$\approx L / \rho^D = O(1)$};
 
% Margin indicator LD (large)
\draw[<->, green!60!black, very thick] (1.4, 7) -- (1.8, 7);
\draw[<->, green!60!black, very thick] (-1.4, 7) -- (-1.8, 7);
\node[right, green!60!black, font=\scriptsize] at (2.0, 7) {$\delta_D \gg 0$};
 
% Contraction Cones qD-1
\begin{pgfonlayer}{background}
     \fill[contraction_cone] (M1.north west) -- (M1.north east) -- (MD.south) -- cycle;
\end{pgfonlayer}
\node[map_arrow, rotate=90] at (-2, 5.5) {$q_{D-1}$};

% --- Logarithmic Depth Label ---
\draw[dim_line] (-7.5, 0) -- (-7.5, 7) node[midway, fill=white, align=center, rotate=90] {Hierarchy Depth $D = O(\log L)$};

% ============================================================
% RIGHT SIDE: SEPARATOR PUMP (parity alternation)
% ============================================================
 
% Title
\node[font=\bfseries, align=center] at (9, 7.8) {Separator Pump\\[-2pt]\footnotesize (Parity Alternation)};
 
% --- Pump cycle between L0 and L1 ---
 
% Odd phase (expansion)
\node[odd_phase] (odd1) at (9, 1.0) {\textbf{Odd} ($\Hodd$)\\[-2pt]\scriptsize Expand / Explore};
\draw[pump_arrow, orange!70!black] (odd1.south) ++(0,-0.05) -- ++(0, -0.5) node[right, font=\scriptsize, orange!70!black, pos=0.5] {expose $\partial$};
 
% Even phase (contraction)  
\node[even_phase] (even1) at (9, 2.5) {\textbf{Even} ($\Heven$)\\[-2pt]\scriptsize Contract / Validate};
\draw[pump_arrow, blue!60!black] (even1.south) ++(0,-0.05) -- ++(0, -0.5) node[right, font=\scriptsize, blue!60!black, pos=0.5] {collapse $\beta_1 \!\to\! \beta_0$};
 
% Cycle arrow
%\draw[pump_arrow, gray!70, thick, dashed] (odd1.north) to[out=160,in=200] (even1.south west);
 
% Delta increase annotation
\node[draw=green!60!black, fill=green!5, rounded corners=2pt, font=\scriptsize, align=center, minimum width=2.4cm] (delta1) at (9, 3.6) {$\delta_0 \to \delta_1$\\margin increases};

% --- Separator line ---
\draw[gray!40, thick, dashed] (7.2, 4.3) -- (10.8, 4.3);
 
% --- Pump cycle between L_{k} and L_{k+1} (abstract) ---
 
\node[odd_phase] (odd2) at (9, 4.9) {\textbf{Odd} ($\Hodd$)\\[-2pt]\scriptsize Expand / Explore};
\draw[pump_arrow, orange!70!black] (odd2.south) ++(0,-0.05) -- ++(0, -0.5);
 
\node[even_phase] (even2) at (9, 6.2) {\textbf{Even} ($\Heven$)\\[-2pt]\scriptsize Contract / Validate};
\draw[pump_arrow, blue!60!black] (even2.south) ++(0,-0.05) -- ++(0, -0.5) node[right, font=\scriptsize, blue!60!black, pos=0.5] {collapse};
 
% Cycle arrow
%\draw[pump_arrow, gray!70, thick, dashed] (odd2.north) to[out=160,in=200] (even2.south west);
 
% Delta increase annotation
\node[draw=green!60!black, fill=green!5, rounded corners=2pt, font=\scriptsize, align=center, minimum width=2.4cm] (delta2) at (9, 7.1) {$\delta_{D-1} \to \delta_D$\\Urysohn activates};

% ============================================================
% CONNECTING ANNOTATIONS
% ============================================================
 
% Arrow from pump to ladder
\draw[->, thick, gray!50, dashed] (odd1.west) -- (5.7, 0.5) node[midway, above, font=\scriptsize, gray, sloped] {identifies};
\draw[->, thick, gray!50, dashed] (even1.west) -- (5.7, 2.8) node[midway, above, font=\scriptsize, gray, sloped] {collapses};
 
% Bottom annotation
\node[align=center, font=\footnotesize, color=red!70!black] at (-8.5, 3.5) {\textbf{Geometric}\\\textbf{Necessity:}\\Trading linear\\width for\\logarithmic\\depth.};
 
% Key insight at bottom
\node[draw=black!40, fill=white, rounded corners=3pt, font=\footnotesize, align=center, text width=5cm] at (9, -0.8) {\textbf{Key:} Each odd-even cycle is one\\pump stroke. Part~1 (contraction) is\\the inner operation of Part~2 (alternation).};
 
\end{tikzpicture}
}
\caption{\textbf{The Urysohn Ladder with Parity Alternation.} \textbf{Left:} A hierarchy of quotient maps $q_k$ progressively compresses the input stream from $\mathcal{M}_0$ to $\mathcal{M}_D$, trading linear diameter for logarithmic depth. At $\mathcal{M}_0$, class manifolds (red, blue) are entangled with margin $\delta_0 \approx 0$. At $\mathcal{M}_D$, classes occupy well-separated tokens with $\delta_D \gg 0$, and Urysohn's separating function $f$ (green dashed line) exists with finite Lipschitz constant. \textbf{Right:} Each level transition is driven by the \emph{separator pump}: a two-stroke parity alternation cycle. The \emph{odd phase} ($\Hodd$, orange) expands the representation to expose boundary failures; the \emph{even phase} ($\Heven$, blue) contracts validated regions into tokens via metric contraction ($\beta_1 \to \beta_0$). Each cycle monotonically increases the margin $\delta$, iteratively manufacturing the disjointness precondition that Urysohn's Lemma requires.}
\label{fig:log_scaling}
\end{figure}

\noindent\textbf{Bounded Capacity Theorem}
%\label{sec:manifold_capacity}
The Separability Theorem guarantees that a valid decision boundary exists after collapse. The practical viability additionally requires that the resulting quotient space is compact - i.e., that storage scales with intrinsic complexity, not stream duration.

\begin{definition}[Effective Capacity Demand]
The \textbf{effective capacity demand} of a representational manifold $\mathcal{M}$ at resolution $\epsilon$ is:
$C_{\mathrm{eff}}(\mathcal{M}) \;\triangleq\; N(\epsilon, \mathcal{M})$
A fixed-capacity system with hardware budget $d$ can stably represent $\mathcal{M}$ only if $C_{\mathrm{eff}}(\mathcal{M}) \le d$.
\end{definition}
 
%\noindent\textbf{The Flat Manifold Problem}
First, we formalize the geometric source of catastrophic interference \cite{mccloskey1989catastrophic}: in the absence of abstraction, representational demand grows linearly with experience.
 
\begin{lemma}[Linear Capacity Growth on Flat Manifolds]
\label{lem:linear_growth}
If the temporal manifold $\mathcal{M}_0$ is isometric to a line segment of length $L$ (i.e., no metric contraction occurs), then for any fixed resolution $\epsilon > 0$:
    $C_{\mathrm{eff}}(\mathcal{M}_0) = \Theta\left(\frac{L}{\epsilon}\right)$.
\end{lemma}

\noindent\textbf{The Recursive Solution: Bounded Capacity Theorem}
Because the representational demand at each level is governed by the covering number of $\mathcal{M}_k$, recursive metric contraction transforms linear growth in capacity demand into logarithmic growth in hierarchical depth \cite{simon1973architecture}.
To formalize the solution, we introduce recursive metric contraction.
 
\begin{definition}[Recursive $\rho$-Compressibility]
\label{def:rho_compressibility}
A sequence of temporal manifolds $\{\mathcal{M}_k\}_{k=0}^D$ is \emph{recursively $\rho$-compressible} if there exists a sequence of quotient maps $q_k: \mathcal{M}_k \to \mathcal{M}_{k+1}$ such that:
    $C_{\mathrm{eff}}(\mathcal{M}_{k+1}) \le \rho^{-1} C_{\mathrm{eff}}(\mathcal{M}_k)$
where $\rho > 1$ is the uniform compression factor determined by the environment's nested structure.
\end{definition}
\begin{theorem}[Bounded Capacity under Recursive Metric Contraction]
\label{thm:bounded_capacity}
Let $(\mathcal M_0, d_0)$ be a metric space representing experience up to time $L$,
and let $N(\epsilon,\mathcal M)$ denote the $\epsilon$-covering number under $d$.
Assume a hierarchical condensation process produces a sequence of quotient spaces
$\mathcal{M}_0 \xrightarrow{q_0} \mathcal{M}_1 \xrightarrow{q_1} \cdots \xrightarrow{q_{D-1}} \mathcal{M}_D$
equipped with quotient metrics $\{d_k\}_{k\ge 1}$.
Assume that there exists $\rho>1$ such that for all $k$,
$N(\epsilon,\mathcal M_{k+1}) \;\le\; \rho^{-1}\, N(\epsilon,\mathcal M_k)$.
Then
$N(\epsilon,\mathcal M_D) \;\le\; \rho^{-D} N(\epsilon,\mathcal M_0)$,
and in particular, if the representational budget satisfies $N(\epsilon,\mathcal M_D)\le d$,
it suffices to take
$D \;\ge\; \left\lceil \log_{\rho}\!\left(\frac{N(\epsilon,\mathcal M_0)}{d}\right)\right\rceil$.
Therefore, bounded representational demand can be maintained for arbitrarily long streams
by logarithmic growth in hierarchy depth.
\end{theorem}
 
\noindent\textbf{Interpretation.}
Theorem~\ref{thm:bounded_capacity} implies that continual learning need not store or traverse an ever-growing trajectory. Instead, inference proceeds on a compact quotient manifold whose geometry encodes long-range temporal structure, operationally replacing linear-time search with short geodesic traversals across a folded manifold \cite{burago2001course}. The connection between geometric folding and classical space-time trade-offs (Savitch-style recursion vs.\ dynamic programming) is developed in the next section.
 
\section{Recursion via Hierarchical Depth}
\label{sec:recursive_depth}

\paragraph{Logarithmic Depth as a Geometric Necessity} 
Theorem \ref{thm:bounded_capacity} establishes that unbounded experience can be represented without increasing the dimensionality of the underlying state space.
It raises a fundamental conservation question: if representational width remains fixed, \emph{where does accumulated complexity reside?}
From a geometric perspective, the answer is neither in width nor in volume, but in \textbf{hierarchical depth} \cite{telgarsky2016benefits,poggio2017theory}.
Next, we show that recursive metric contraction necessarily induces a logarithmic hierarchy of quotient manifolds, which relates geometric depth to space-time trade-offs.
Under the quotient hierarchy $\mathcal{M}_0 \xrightarrow{q_0} \cdots \xrightarrow{q_D} \mathcal{M}_D$ established (Fig. \ref{fig:log_scaling}), each contraction reduces the effective diameter by a factor $\rho > 1$. After $D$ levels, the effective diameter satisfies
$\operatorname{diam}(\mathcal{M}_D) \approx \frac{L}{\rho^D}$.
To maintain a bounded covering number, and hence bounded capacity demand, the hierarchy depth must scale as
$D = O(\log L)$.
The logarithmic growth is a geometric necessity \cite{savitch1970relationships}: any process that maps an unbounded trajectory into a compact representation via repeated contraction must trade linear diameter growth for logarithmic depth growth. Hierarchical depth is the \emph{geometric price} paid for constant-width inference (i.e., increasing depth does not increase instantaneous capacity demand), since each level operates on a compact quotient with bounded covering number.

\begin{comment}
\begin{theorem}[Hierarchical Scaling Law (Quotient-Distance Scaling)]
\label{thm:hier_scaling}
Let $\Omega$ be a state space (graph or metric space) and let $q:\Omega\to Q$ be a quotient map inducing a quotient graph $G_Q=(Q,E_Q)$.
Assume there exists a constant $\xi>0$ such that for any quotient cell $q_i\in Q$, the cost of planning \emph{within} $q_i$ (to enter and exit $q_i$ through any admissible boundary states) is bounded by $\xi$.
Let $s,g\in\Omega$ and let $q(s),q(g)\in Q$ denote their quotient classes.
Then the total planning energy satisfies
$E_{\mathrm{total}}(s,g)\le
\xi \cdot d_{G_Q}\!\bigl(q(s),q(g)\bigr)+E_{\mathrm{hi}}\!\bigl(q(s),q(g)\bigr)+E_{\mathrm{bdry}}$,
where $d_{G_Q}$ is shortest-path distance on $G_Q$, $E_{\mathrm{hi}}$ is the energy/cost to find a path on $G_Q$, and $E_{\mathrm{bdry}}$ is any (typically constant-factor) overhead for boundary-crossings.
In particular, if $E_{\mathrm{hi}}(q(s),q(g))=O\!\bigl(d_{G_Q}(q(s),q(g))\bigr)$ (e.g., sparse $G_Q$ and efficient shortest-path search),
then
$E_{\mathrm{total}}(s,g)=\Theta\!\Bigl(d_{G_Q}\!\bigl(q(s),q(g)\bigr)\Bigr)$.
Moreover, in the worst case over all start-goal pairs,
$\sup_{s,g\in\Omega} E_{\mathrm{total}}(s,g)=O\!\bigl(\xi\cdot \mathrm{diam}(G_Q)\bigr)+\sup_{u,v\in Q}E_{\mathrm{hi}}(u,v)+E_{\mathrm{bdry}}$.
\end{theorem}
\end{comment} 
 
\paragraph{Geometric Relation to Savitch-Style Recursion}
 
The depth-width trade-off admits a direct geometric interpretation of classical results in computational complexity, most notably \textbf{Savitch’s Theorem} \cite{savitch1970relationships}.
Savitch’s construction shows that nondeterministic space-bounded reachability can be simulated deterministically by recursively subdividing paths, trading space for recursion depth.
In geometric terms, Savitch-style recursion corresponds to inference on a \emph{flat} manifold:
long trajectories are verified by repeatedly bisecting geodesics, without altering the underlying geometry.
Although space usage remains bounded, inference repeatedly revisits the same regions of the manifold, incurring exponential time overhead.
Recursive condensation implements a geometric analog of memoization.
When a sub-trajectory is validated and collapsed via a quotient map, its internal geometry is removed from future consideration.

\begin{lemma}[Compatibility $\Rightarrow$ Local Boundedness]
\label{lem:compat_local_bound}
Let $(X,d)$ be a metric space and let $q:X\to Y:=X/{\sim}$ be a quotient map with equivalence classes
$[x]=q^{-1}(q(x))$. Assume:
1) \textbf{Bounded class diameter (metric collapse):} there exists $\varepsilon>0$ such that
$\mathrm{diam}([x]) := \sup_{u,v\in[x]} d(u,v) \le \varepsilon
\quad \forall x\in X$;
2) \textbf{Compatible separator:} there exists a continuous separator $f:X\to[0,1]$ that is constant on equivalence
classes, i.e.,
$u\sim v \ \Rightarrow\ f(u)=f(v)$;
3) \textbf{Quantitative regularity (Lipschitz):} $f$ is $L$-Lipschitz on $(X,d)$:
$|f(u)-f(v)| \le L\,d(u,v)\quad \forall u,v\in X$.
Then there exists a unique continuous function $\bar f:Y\to[0,1]$ such that $f=\bar f\circ q$.
Moreover, for any $x\in X$ and any $u\in[x]$,
$|f(u)-f(x)| \le L\,\varepsilon$,
so the variation of the score inside each quotient cell is uniformly bounded.
\end{lemma}

Subsequent inference does not re-traverse the original geodesic; it traverses a \emph{wormhole}, a metric shortcut induced by contraction, which transforms repeated recursive verification into direct navigation on a folded manifold.
From a wormhole perspective, we observe:
1) \textbf{Flat inference (search)} corresponds to geodesic traversal on $\mathcal{M}_0$, analogous to Savitch-style recursion without memoization;
2) \textbf{Condensation} corresponds to collapsing verified geodesic segments into singular points, eliminating redundant distance;
3) \textbf{Hierarchical inference} corresponds to navigation on $\mathcal{M}_D$, where long temporal separations have been converted into short paths through quotient geometry.
The resulting structure is a \textbf{tower of quotient manifolds}, which we call Urysohn Ladder (Fig. \ref{fig:log_scaling}), in which complexity is serialized into depth rather than width.
Just as Savitch’s theorem manages space by serializing verification, Urysohn Ladder manages capacity by recursive condensation.
 
\begin{theorem}[Recursive separation preserved by quotienting]
\label{thm:urysohn_quotient}
Let $X_0$ be normal and $A_0,B_0\subset X_0$ disjoint closed as in Theorem \ref{thm:topological_collapse}.
Let $q_k:X_k\to X_{k+1}$ be quotient maps and define $A_{k+1}:=q_k(A_k)$, $B_{k+1}:=q_k(B_k)$.
If there is a continuous $f_0:X_0\to[0,1]$ separating $A_0,B_0$ such that, for each $k$,
$x\sim_k x' \implies f_k(x)=f_k(x'),
\quad\text{where } f_k \text{ is the descended separator on } X_k$,
then for every level $k$ there is a continuous $f_k:X_k\to[0,1]$ separating $A_k$ and $B_k$.
\end{theorem}
 
%\begin{remark}[The Dimensionality Paradox] Theorem~\ref{thm:urysohn_quotient} highlights a shift from geometric estimation to topological construction. The constructive Urysohn separator relies on distance ratios $f(x) \propto d(x, A) / [d(x, A) + d(x, B)]$, which are independent of the ambient dimension $D$. By enforcing metric collapse via $q_k$, the system reaches a regime where high dimensionality provides ample degrees of freedom to route $f_k$ without intersection --- dimensionality becomes an asset for routing rather than a curse for estimation. \end{remark}

\section{Empirical Validation}
\label{sec:empirical}
 
We validate each core theorem on real-world data and pretrained models. For each experiment, we state the theoretical prediction, describe the protocol, and report quantitative results. Additional experimental results on the scalability of parity alternation in Urysohn Ladder for continual learning (Theorem \ref{thm:urysohn_quotient}) and discussion about limitations of this work can be found in the Appendix. %Full implementation details and code are provided in the supplementary material.

\paragraph{Experiment 1: Topological Collapse Separability (Theorem~\ref{thm:topological_collapse})}
%\emph{Prediction.}
Stronger metric collapse (lower within-class / between-class distance ratio) should monotonically improve linear separability without any increase in embedding dimension.
%\emph{Protocol.}
We train a small embedding network (3 conv layers, $d_{\mathrm{embed}}{=}64$) with supervised contrastive loss \cite{khosla2020supervised} on MNIST and CIFAR-10, sweeping temperature $\tau \in \{2.0, 1.0, 0.5, 0.25, 0.1, 0.05\}$. Lower $\tau$ produces harder within-class contraction. At each setting, we measure: (i) linear-probe accuracy (logistic regression on learned embeddings), and (ii) the collapse ratio (mean within-class distance / mean between-class centroid distance).
%\emph{Results (Fig.~\ref{fig:exp1}).}
As shown in Fig.~\ref{fig:exp1}, linear-probe accuracy improves monotonically with collapse strength on both datasets, reaching $>99\%$ on MNIST and $>75\%$ on CIFAR-10 at $\tau{=}0.05$. The collapse ratio decreases in tandem, confirming that within-class distances contract faster than between-class distances separate, which validates Theorem~\ref{thm:topological_collapse}: metric contraction alone enables linear separability - the kernel trick is unnecessary when the quotient map has sufficient strength.
 
\begin{figure}[h]
    \centering
    \includegraphics[width=\linewidth]{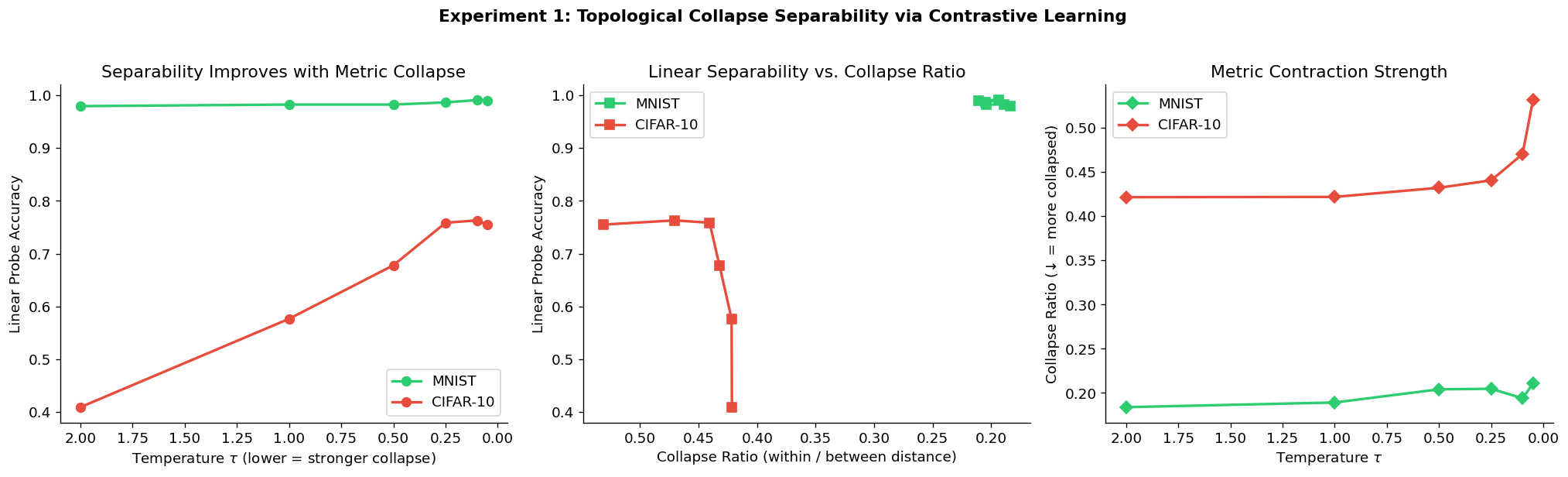}
    \caption{Supervised contrastive learning as a learned quotient map \textbf{(Experiment 1)}. Linear-probe accuracy (left) and collapse strength (right) improve monotonically with $\tau \downarrow$) on both MNIST and CIFAR-10, verifying that metric contraction enables linear separability (middle, Theorem~\ref{thm:topological_collapse}).}
    \label{fig:exp1}
\end{figure}
 
\paragraph{Experiment 2: Parity-Partitioned Stability (Theorem~\ref{thm:parity_stability})}
 
%\emph{Prediction.}
Decomposing parameters into orthogonal flow/scaffold subspaces with alternating updates should eliminate cross-task interference entirely.
%\emph{Protocol.}
We implement parity partitioning in a 4-layer MLP on Permuted-MNIST (5 sequential tasks). Even-indexed layers (0, 2) form the scaffold $\theta_S$, frozen after initial training on Task~0. Odd-indexed layers (1, 3) plus the classification head form the flow $\theta_F$, re-initialized per task. Each task's flow parameters are stored as a per-task snapshot (topological isolation via disjoint fibers). We compare against fine-tuning (no protection) and EWC ($\lambda{=}500$) \cite{kirkpatrick2017overcoming}.
%\emph{Results (Fig.~\ref{fig:exp2}).}
As shown in Fig.~\ref{fig:exp2}, fine-tuning exhibits catastrophic forgetting: Task~0 accuracy drops to ${\sim}20\%$ (chance) after learning subsequent tasks. EWC mitigates the forgetting partially, but average accuracy still degrades to ${\sim}70\%$ after 5 tasks. Parity partitioning achieves zero forgetting by construction - the accuracy matrix confirms that evaluating any prior task with its stored flow snapshot recovers the original performance exactly (diagonal entries $>95\%$). The trade-off is $O(T)$ memory for flow snapshots, with efficiency governed by the scaffold-to-flow parameter ratio.

\begin{figure}[h]
    \centering
    \includegraphics[width=\linewidth]{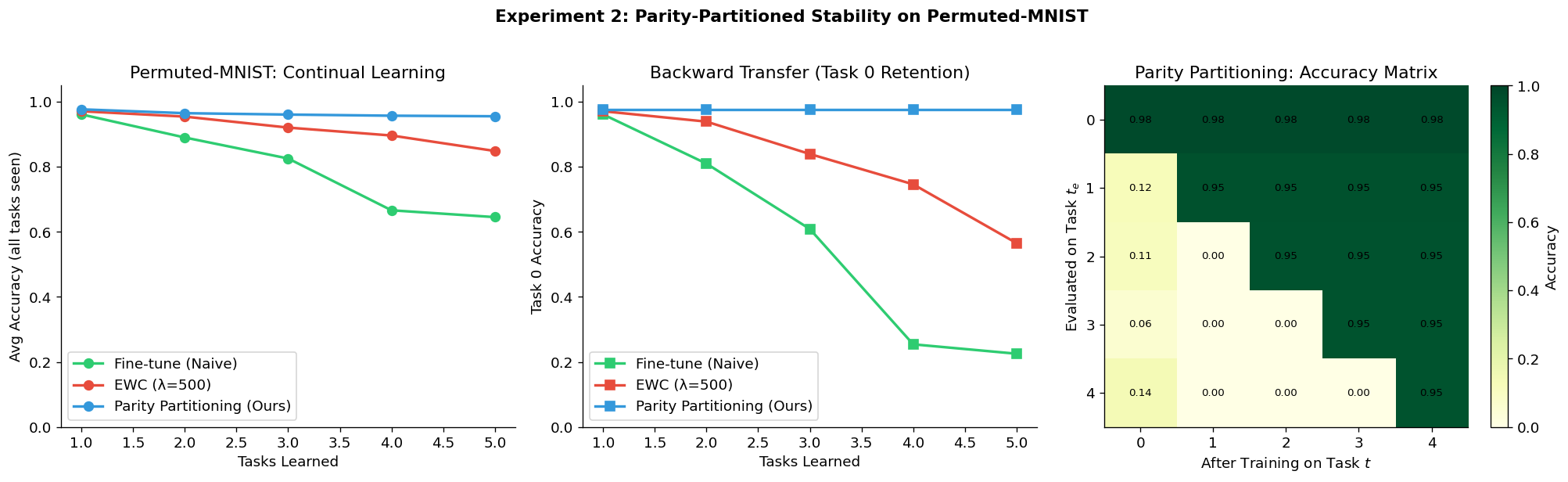}
    \caption{Continual learning on Permuted-MNIST \textbf{(Experiment 2)}. Left- Avg. Accuracy Performance; middle- Task 0 retention rate; right - Accuracy matrix. Fine-tuning (green) suffers catastrophic forgetting; EWC (red) partially mitigates it; parity partitioning (blue) achieves zero forgetting by isolating per-task flow parameters while sharing a frozen scaffold (Theorem~\ref{thm:parity_stability}).}
    %\vspace{-0.2in}
    \label{fig:exp2}
\end{figure}

\paragraph{Experiment 3: Bounded Capacity (Theorem~\ref{thm:bounded_capacity})}
 
%\emph{Prediction.}
On a flat manifold, the $\epsilon$-covering number grows as $N(\epsilon) \propto L$. Under recursive $\rho$-contraction, $N(\epsilon, \mathcal{M}_D) \leq \rho^{-D} N(\epsilon, \mathcal{M}_0)$, yielding bounded capacity at logarithmic depth.
%\emph{Protocol.}
We extract token-level hidden states ($d{=}768$) from GPT-2 over progressively longer contexts ($L \in \{64, 128, 256, 512, 768, 1024\}$) and compute empirical covering numbers via greedy $\epsilon$-nets ($\epsilon{=}6.0$). We then apply a BPE-like hierarchical compression proxy that mimics quotient maps: at each level, adjacent hidden states with minimal distance are merged (averaged), reducing sequence length by factor $\rho \approx 2.0$ per level, up to 8 levels.
%\emph{Results (Fig.~\ref{fig:exp3}).}
Fig.~\ref{fig:exp3} shows the result of our scaling study. Flat covering numbers grow approximately linearly with context length (slope $\approx 0.35$), confirming Lemma~\ref{lem:linear_growth}. After hierarchical compression, covering numbers at each level decay exponentially, closely tracking the theoretical bound $N_0 / \rho^k$. The top-level covering number remains approximately constant across all input lengths, directly validating Theorem~\ref{thm:bounded_capacity}.

\begin{figure}[h]
    \centering
    \includegraphics[width=\linewidth]{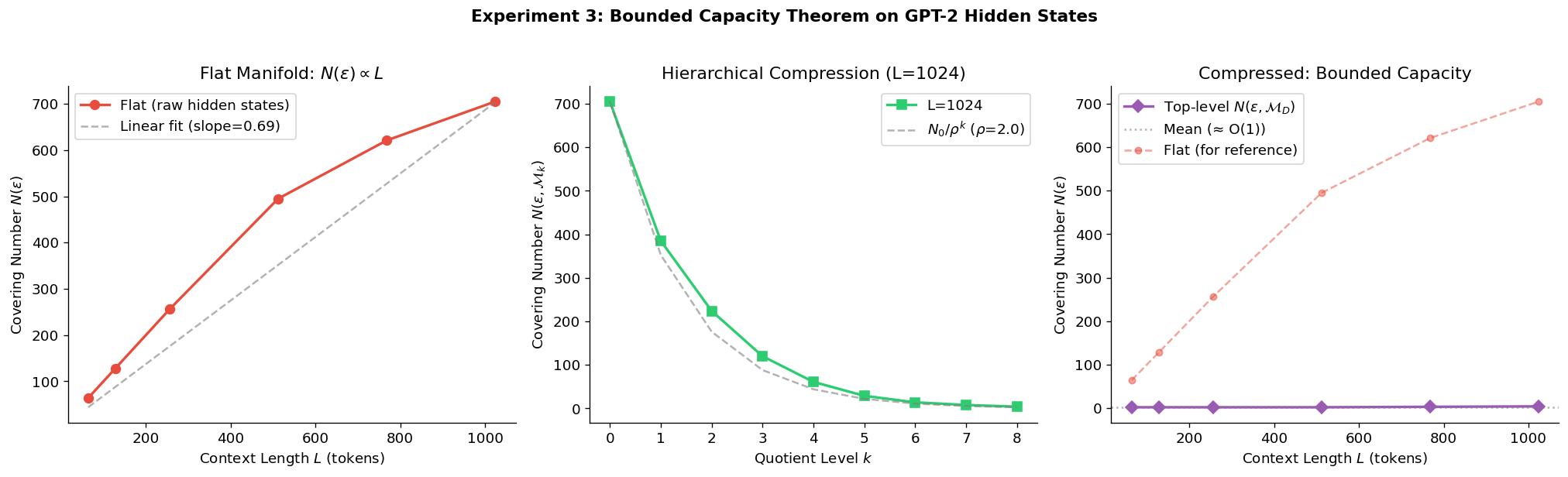}
    \caption{Covering number scaling on GPT-2 hidden states \textbf{(Experiment 3)}. Flat-manifold covering numbers grow linearly with context length $L$ (red), while BPE-like hierarchical compression (green) yields bounded top-level covering (purple) regardless of $L$, confirming bounded capacity of Urysohn Ladder as predicted by Theorem~\ref{thm:bounded_capacity}.}
    
    \label{fig:exp3}
\end{figure}

\vspace{-0.2in} 
\section{Conclusion}
\vspace{-0.1in}
We introduced a geometric account of continual learning that treats abstraction as a \emph{physical deformation} of representational space. The core obstacle to lifelong learning is the \emph{flat manifold problem}: when experience is encoded as an expanding trajectory in a rigid metric, geodesic distances and covering numbers grow with time, forcing overlap and catastrophic interference. Our framework resolves the plasticity-stability paradox by \emph{recursive metric contraction} implemented by the Urysohn Ladder as a hierarchy of quotient maps that folds the temporal manifold into a compact scaffold. In our view, learned tokens are not merely compressed features, but \emph{geodesic shortcuts}, regions of extreme metric contraction that make distant events adjacent in the induced quotient geometry. Empirical validation across four experiments on real-world data and pretrained models confirms the theoretical predictions: bounded capacity under hierarchical compression, linear separability through metric collapse, interference-free continual learning via parity partitioning, and progressive topological disentanglement through network depth. 
 
\section*{Broader Impact Statement}
 
This work proposes a geometric framework for continual learning in which abstraction is realized as recursive metric contraction (quotienting) that converts repeated ambient-space search into low-action navigation on an intrinsic scaffold. As demonstrated by our empirical results, the primary impact is to shift how the community thinks about scalability in lifelong learning: from increasing parameter counts or replay budgets to actively \emph{reshaping} representational geometry so that capacity, interference, and inference cost are governed by intrinsic structure rather than ambient dimension. The framework also suggests a pathway toward \emph{energy-efficient} ML by separating ``write'' costs (rare structural commitments) from ``read/routing'' costs (cheap traversal), which could reduce both compute expenditure and carbon footprint for deployment on edge and neuromorphic hardware.
 
Potential negative impacts are limited but include the risk of over-interpreting topological metaphors as biological claims without sufficient neurophysiological evidence, and the possibility that aggressively compressed/tokenized representations could amplify dataset biases if the contracted scaffold preserves spurious correlations as stable ``shortcuts.'' We encourage careful benchmarking on diverse continual-learning settings, explicit bias audits across tasks and subpopulations, and transparent reporting of failure modes (e.g., when quotienting harms rare concepts). Overall, we expect the net impact to be positive by providing a principled lens for designing scalable, stable, and potentially lower-energy learning systems, while highlighting concrete assumptions that can be falsified experimentally.

\bibliography{bibfile,references}

%%%%%%%%%%%%%%%%%%%%%%%%%%%%%%%%%%%%%%%%%%%%%%%%%%%%%%%%%%%%
\newpage
\appendix

\renewcommand{\theequation}{A\arabic{equation}}
\renewcommand{\thefigure}{A\arabic{figure}}
\renewcommand{\thetable}{A\arabic{table}}
\setcounter{equation}{0}    
\setcounter{figure}{0}    
\setcounter{table}{0}    

% \onecolumn
\section{Proofs of Four Theorems}
\label{appendix:A}

\begin{proof}[Proof of Theorem \ref{thm:topological_collapse}]
Existence of $f$ follows from Urysohn's lemma.
By construction, $f$ is constant on $\sim_f$-equivalence classes, hence $\bar f([x]):=f(x)$ is well-defined.
The quotient universal property implies $\bar f$ is continuous and satisfies $f=\bar f\circ q$.
Since $f\equiv 0$ on $A$ and $f\equiv 1$ on $B$, each set is contained in a single equivalence class, so $q(A)$ and $q(B)$ are singleton points with $\bar f$-values $0$ and $1$.
\end{proof}

\begin{proof}[Proof of Theorem \ref{thm:parity_stability}]
\textbf{Step 1: Homological Orthogonality.}
The Parity Principle posits that $\mathcal{M}_{odd}$ represents the space of 1-cycles (active flows) and $\mathcal{M}_{even}$ represents the space of 0-cycles (connected components/points). In algebraic topology, these distinct homology groups ($H_1$ and $H_0$) are generators of orthogonal subspaces in the total phase space. No continuous deformation within the connected component of a cycle (in $H_1$) alters the position of disconnected components (in $H_0$), provided the boundary condition $\partial \gamma = 0$ is met.
 
\textbf{Step 2: The Parity-Inverting Map.}
Let the learning operator be $\Psi(\gamma_t) = p_{new}$, where $\gamma_t \subset \mathcal{M}_{odd}$ and $p_{new} \in \mathcal{M}_{even}$.
The Jacobian of this transformation describes how changes in the input affect the memory.
Since $\mathcal{M}_{odd}$ and $\mathcal{M}_{even}$ are disjoint during the active inference phase (separated by the wake/sleep or search/navigation phase transition), the cross-term of the metric tensor is zero.
 
\textbf{Step 3: Invariance.}
Let $S_{old} \subset \mathcal{M}_{even}$ be the set of previously learned concepts.
The collapse of a new trajectory $\gamma_{new}$ involves a contraction of the metric $g_{odd}$ in the flow space.
Since $S_{old}$ resides entirely in the metric space $g_{even}$, and the coupling is unidirectional ($\Psi: odd \to even$), the contraction of $g_{odd}$ does not induce stress or displacement on the geodesic distances in $g_{even}$.
 
\textbf{Step 4: Conclusion.}
The capacity of $\mathcal{M}_{even}$ to store linearly separable points is limited only by the volume of $\mathcal{M}_{even}$ (Theorem 1) and is \textit{independent} of the dynamic complexity required to generate those points in $\mathcal{M}_{odd}$. Therefore, the system effectively avoids catastrophic interference.
\end{proof}

\begin{proof}[Proof of Theorem \ref{thm:bounded_capacity}]
By the \textbf{Compressibility} assumption, for every level $k\ge 0$ we have
\begin{equation}
\label{eq:compress_step}
N(\epsilon,\mathcal M_{k+1}) \;\le\; \rho^{-1}\, N(\epsilon,\mathcal M_k).
\end{equation}
Applying \eqref{eq:compress_step} repeatedly yields a telescoping bound. Concretely,
$N(\epsilon,\mathcal M_{1}) \le \rho^{-1} N(\epsilon,\mathcal M_0)$,
$N(\epsilon,\mathcal M_{2}) \le \rho^{-1} N(\epsilon,\mathcal M_1)\le \rho^{-2} N(\epsilon,\mathcal M_0)$,
and by induction, assume $N(\epsilon,\mathcal M_{k}) \le \rho^{-k} N(\epsilon,\mathcal M_0)$. Then
\begin{equation}
\begin{aligned}
N(\epsilon,\mathcal M_{k+1})
&\le \rho^{-1} N(\epsilon,\mathcal M_k) \\
&\le \rho^{-1}\cdot \rho^{-k} N(\epsilon,\mathcal M_0)
= \rho^{-(k+1)} N(\epsilon,\mathcal M_0).
\end{aligned}
\end{equation}
Therefore, for all $D\ge 0$,
$N(\epsilon,\mathcal M_D) \;\le\; \rho^{-D} N(\epsilon,\mathcal M_0)$,
which proves the first claim.
 
For the budgeted statement, suppose we require $N(\epsilon,\mathcal M_D)\le d$ for some representational budget $d\in\mathbb N$.
A sufficient condition is
$\rho^{-D} N(\epsilon,\mathcal M_0) \;\le\; d$.
Rearranging gives
$\rho^D \;\ge\; \frac{N(\epsilon,\mathcal M_0)}{d}$.
Taking $\log_\rho(\cdot)$ on both sides (noting $\rho>1$) yields
$D \;\ge\; \log_{\rho}\!\left(\frac{N(\epsilon,\mathcal M_0)}{d}\right)$.
Since $D$ must be an integer depth, it suffices to choose
$D \;\ge\; \left\lceil \log_{\rho}\!\left(\frac{N(\epsilon,\mathcal M_0)}{d}\right)\right\rceil$,
as claimed. Hence, when $N(\epsilon,\mathcal M_0)$ grows with stream length (e.g., $L\to\infty$),
the required depth grows only logarithmically in the initial covering number, establishing that bounded representational demand can be maintained by increasing hierarchy depth.
\end{proof}
 
\begin{proof}[Proof of Theorem \ref{thm:urysohn_quotient}]
We proceed by induction on the hierarchy level $k$.
\textbf{Base Case ($k=0$):}
Since $X_0$ is normal and $A_0,B_0\subset X_0$ are disjoint closed sets, Urysohn's Lemma yields a continuous function
$f_0:X_0\to[0,1]$ such that $f_0(A_0)=\{0\}$ and $f_0(B_0)=\{1\}$.
 
\textbf{Inductive Step:}
Assume there exists a continuous $f_k:X_k\to[0,1]$ with $f_k(A_k)=\{0\}$ and $f_k(B_k)=\{1\}$.
Let $q_k:X_k\to X_{k+1}=X_k/{\sim_k}$ be the quotient map, and assume compatibility:
$x\sim_k x' \implies f_k(x)=f_k(x')$.
Then $f_k$ is constant on equivalence classes, so the function
\[
f_{k+1}:X_{k+1}\to[0,1],\qquad f_{k+1}(q_k(x)):=f_k(x)
\]
is well-defined. Moreover, by the universal property of quotient maps, $f_{k+1}$ is continuous and satisfies
$f_{k+1}\circ q_k=f_k$.
 
\textbf{Separation at level $k+1$:}
For $y\in A_{k+1}=q_k(A_k)$, pick $x\in A_k$ with $q_k(x)=y$. Then
$f_{k+1}(y)=f_k(x)=0$.
Similarly, $f_{k+1}(y)=1$ for all $y\in B_{k+1}=q_k(B_k)$.
In particular, compatibility prevents $q_k$ from identifying any $a\in A_k$ with $b\in B_k$ (otherwise $0=f_k(a)=f_k(b)=1$), so $A_{k+1}$ and $B_{k+1}$ remain disjoint.
Therefore, we have $f_{k+1}$ separates $A_{k+1}$ and $B_{k+1}$. By induction, the claim holds for all $k$.
\end{proof}

\section{Related Work}
\label{sec:related}
 
\noindent\textbf{Capacity limits in fixed-dimensional representations.}
Classical statistical learning theory establishes sharp limits on the capacity of fixed-dimensional hypothesis spaces.
Cover’s seminal result shows that the probability of linear separability of random patterns in $\mathbb{R}^d$ collapses once the number of points exceeds $O(d)$ \cite{cover1965geometry}.
This phenomenon is formalized more generally through the Vapnik-Chervonenkis (VC) dimension and related notions of covering numbers, which bound the number of distinguishable regions that can be stably represented in a given space \cite{vapnik1998statistical,shalev2014understanding}.
From a geometric perspective, these results imply that maintaining separability over an ever-lengthening trajectory in a flat manifold inevitably exhausts capacity, independent of the learning algorithm.
 
\noindent\textbf{Catastrophic interference as geometric overload.}
The breakdown of performance under sequential learning has long been recognized as catastrophic interference or catastrophic forgetting \cite{mccloskey1989catastrophic,french1999catastrophic}.
In connectionist models, gradient-based updates applied to a shared representational space cause new trajectories to overwrite previously learned structure unless explicit safeguards are imposed \cite{robins1995catastrophic}.
Contemporary continual learning methods address such failure mode through rehearsal buffers \cite{schaul2015prioritized}, parameter regularization schemes such as elastic weight consolidation \cite{kirkpatrick2017overcoming}, or architectural growth.
While effective empirically, these approaches implicitly assume a flat underlying geometry and treat interference as a resource-allocation problem rather than as a consequence of unbounded metric expansion.
 
\noindent\textbf{Hierarchical abstraction and depth as a control mechanism.}
A complementary body of work emphasizes hierarchical organization as a means of controlling complexity.
Simon argued that adaptive systems exploit hierarchical decomposition to remain tractable in complex environments \cite{simon1973architecture}.
In reinforcement learning, hierarchical methods such as feudal RL \cite{dayan1993feudal} and MAXQ \cite{dietterich2000maxq} reduce planning and learning complexity by introducing temporally extended abstractions.
From a representational standpoint, theoretical analyses have shown that depth enables efficient reuse of intermediate structure, allowing expressive power to grow exponentially with only linear increases in parameters \cite{telgarsky2016benefits}.
These results suggest that abstraction depth, rather than representational width, plays a central role in scalability.

\noindent\textbf{Geometric and topological perspectives on representation learning.}
Recent work has increasingly framed learning and abstraction through the lens of geometry.
The manifold hypothesis posits that high-dimensional data concentrate near low-dimensional structures, motivating geometric methods for representation learning and motivating tests of whether data admit manifold structure at all \cite{tenenbaum2000global,fefferman2016testing}.
Spectral, diffusion, and covering-based approaches further connect the geometry of data manifolds to compact representations through eigenstructure, Laplacian operators, and multiscale decompositions \cite{chung1997spectral,mahadevan2005proto,belkin2008towards}.
A related line of work studies \emph{grokking}, the phenomenon in which models generalize only long after they have already fit the training set, and often interprets this delayed transition as a late reorganization of the learned representation rather than as simple interpolation \cite{power2022grokking,liu2022understanding,kumar2023grokking}.
From this viewpoint, grokking suggests that successful learning may require not only fitting within an already available geometric chart, but also discovering a more appropriate global organization of the task manifold itself.
However, most existing geometric methods still treat the underlying manifold as fixed and focus on embedding, smoothing, or parametrizing it, rather than actively deforming its metric or topology to create a representation in which the task becomes structurally simpler.
 
\paragraph{Relation to Graduated Non-Convexity (GNC)}
The dynamic process of metric contraction can be conceptually interpreted as the asymptotic limit of Graduated Non-Convexity (GNC) \cite{blake1987visual}. In standard GNC, a non-convex energy functional $E$ is approximated by a sequence of functionals $E^{(p)}$ ranging from a convex approximation ($p=1$) to the true non-convex energy ($p=0$).
In our recursive quotienting network (RQN), the Parity Alternation Principle naturally implements a GNC schedule:
1) During the \textit{Search Phase} (Expansion), the system operates in a high-entropy regime effectively smoothing the energy landscape, allowing the state trajectory to escape local minima (hallucinations) and sense the global homological structure;
2) During the \textit{Collapse Phase} (Condensation), the system lowers the temperature, sharpening the energy wells.
Unlike classical GNC, which seeks a coordinate solution $x^* \in \mathbb{R}^n$, the RQN metric collapse seeks a topological quotient. As the stability parameter $\sigma \to 0$, the metric tensor $g_{ij}$ becomes singular along the concept manifold, effectively identifying the entire basin of attraction as a single point in the quotient space $\mathcal{M}/\sim$. Thus, Metric Collapse is the \textit{topological endpoint} of the GNC trajectory.
 
\noindent\textbf{Connection to Biological Implementations.}
The form of metric contraction formalized above is not merely a mathematical abstraction, but is supported by well-established biological learning mechanisms \cite{muller1996hippocampus,buzsaki2013memory}.
In cortical and hippocampal circuits, repeated co-activation of neural populations induces synaptic strengthening that effectively shortens functional distances between states, a phenomenon traditionally described in terms of attractor formation, representational compression, or chunking.
From a geometric perspective, these processes implement a local contraction of the representational metric: states that reliably co-occur or predict one another become separated by progressively shorter geodesic distances in neural state space.
Metric contraction operates across multiple spatial and temporal scales.
At the synaptic and population level, \emph{generalized Hebbian learning} (GHL), including Oja- and Sanger-type rules \cite{oja1982simplified,sanger1989optimal}, aligns neural representations with the dominant low-dimensional structure of experienced trajectories, effectively flattening and unrolling curved manifolds in activity space.
Such learning rules can be interpreted as adapting the local metric tensor to emphasize directions of consistent variance while collapsing redundant degrees of freedom.
At a higher organizational level, hippocampal replay during offline states (e.g., sharp-wave ripples \cite{wilson1994reactivation}) repeatedly reactivates extended temporal sequences in compressed form, accelerating traversal through learned state sequences and reinforcing long-range associations.
Geometrically, these biological mechanisms collectively implement a quotient-map abstraction.
Extended temporal trajectories are progressively collapsed into compact neural assemblies whose internal transitions are traversed rapidly and with minimal metabolic cost \cite{hasson2015hierarchical}.
 
\noindent\textbf{Positioning of the present work.}
The present work isolates a distinct mechanism for AGI scalability: \emph{recursive metric contraction}.
Rather than expanding representational dimension (Kernel methods), replaying past data (Generative Replay), or learning a static embedding (Pre-training), we formalize abstraction as a sequence of metric-topology factorization (MTF) and topological quotient operations that collapse validated sub-manifolds of experience (Fig. \ref{fig:log_scaling}).
We substantiate the present framework with four guarantees:
\begin{itemize}
    \item \textbf{Separability:} Our \emph{Topological Collapse Separability Theorem} (Theorem~\ref{thm:topological_collapse}) positions metric quotienting as the geometric dual to the Kernel Trick: non-linear problems can be solved not by adding dimensions, but by collapsing the metric until the problem becomes topologically trivial. Our \emph{Recursive Separation Theorem} (Theorem~\ref{thm:urysohn_quotient}) further establishes that this separability propagates faithfully through the entire quotient hierarchy, guaranteeing that disjoint concepts remain topologically separable regardless of abstraction depth.
    \item \textbf{Bounded Capacity:} Our bounded capacity theorem (Theorem~\ref{thm:bounded_capacity}) guarantees that the effective capacity demand remains uniformly bounded by $O(1)$, independent of stream length.
    \item \textbf{Scalability:} Our \emph{Hierarchical Scaling Law} (Theorem~\ref{thm:urysohn_quotient}) proves that total planning/inference cost scales with quotient geodesic distance rather than ambient domain size, yielding \emph{bounded peak load} plus \emph{quotient-distance scaling} of end-to-end effort.
    \item \textbf{Stability:} Our \emph{Parity-Partitioned Stability Theorem} (Theorem~\ref{thm:parity_stability}) shows that by segregating the manifold into orthogonal flow ($\mathcal{M}_{odd}$) and scaffold ($\mathcal{M}_{even}$) subspaces, the system supports infinite plasticity without catastrophic interference.
\end{itemize}
\noindent
Collectively, these results reframe the central problem of continual learning: it is not an unavoidable consequence of finite memory resources, but a failure to dynamically transform the topology of the representational manifold.

\paragraph{Limitations of the present work.}
We identify several limitations of the current work that suggest directions for future research.

\textit{1) Quotient construction depends on metric quality.}
The Urysohn Ladder's contraction guarantees (Theorems 1-2) assume that the base metric space captures semantically meaningful distances: metrically close points should be functionally equivalent, and metrically distant points should be functionally distinct.
In practice, this assumption is offloaded to the pretrained feature extractor, which may not satisfy it uniformly across all downstream domains.
When the pretrained features are poorly calibrated for a particular data distribution, for example, fine-grained distinctions that the pretraining objective did not incentivize, the quotient maps may collapse semantically distinct inputs into the same token (false merges) or fail to collapse equivalent inputs (redundant tokens), degrading both the separability and capacity guarantees.
Characterizing the relationship between pretraining distribution, downstream task structure, and quotient fidelity remains an open theoretical question.

\textit{2) Validation oracle for neighborhood collapse.}
The recursive contraction procedure of MTF requires a validation signal to determine when a metric neighborhood has converged sufficiently to be collapsed into a token.
In our experiments, this signal is derived from classification accuracy on held-out data within each neighborhood.
This introduces a dependency on labeled validation data at each quotient level, which may not be available in fully online or unsupervised continual learning settings.
Replacing the supervised validation oracle with self-supervised or uncertainty-based collapse criteria is a natural extension but is not addressed in the current work.

\textit{3) Scaffold amortization at scale.}
While we demonstrate scaffold amortization as a proof of concept, the current implementation operates on MNIST-scale benchmarks where the quotient hierarchy has modest depth.
Whether the amortized scaffold remains efficient when the number of quotient levels grows, as would be required for complex, high-dimensional task distributions, depends on the rate at which covering numbers decay across levels, which we bound theoretically ($O(1)$ per level) but have not yet verified empirically at scale with large pretrained models.
The engineering challenge of maintaining and querying a deep quotient hierarchy in the latent space of a billion-parameter transformer remains unexplored.

\section{Scalability of Parity-MTF: Scaffold Amortization in Continual Learning}
\label{sec:scalability_parity_mtf}

The preceding sections establish that parity partitioning (Theorem~\ref{thm:parity_stability}) provides exact interference-free learning through MTF (orthogonal scaffold-flow decomposition), and that recursive quotienting (Theorem~\ref{thm:bounded_capacity}) bounds representational demand independently of stream length.
A natural question arises: \emph{how do these guarantees compose as the number of tasks grows large?}
In this section, we investigate the scalability of \textbf{Parity-MTF}, a method that combines the structural decoupling of (Metric-Topological Factorization MTF) with the scaffold-flow decomposition prescribed by the Parity Alternation Principle.

\subsection{Method: Parity-MTF}

\paragraph{Architecture.}
Let $P$ denote the total parameter count of a base classifier.
In standard MTF, each detected context receives a full expert (a fresh copy of the base classifier), yielding a total storage cost of $T \cdot P$ for $T$ contexts.
Parity-MTF decomposes the base classifier into two parameter-disjoint subsets of equal total capacity:
\begin{itemize}
    \item \textbf{Scaffold} $\theta_S$ (even-parity layers): Middle layers encoding abstract, permutation-invariant features.
    These are trained on the first task and frozen permanently. Parameter count: $P_S$.
    \item \textbf{Flow} $\theta_F$ (odd-parity layers + head): Input-facing and output-facing layers that must adapt to each task's specific input distribution.
    These are re-initialized and trained independently per task. Parameter count: $P_F = P - P_S$.
\end{itemize}

After training on each task, the flow parameters are snapshotted and stored; the scaffold remains shared across all tasks.
At evaluation, the appropriate flow snapshot is loaded while the scaffold stays fixed.
This guarantees zero forgetting by construction (Theorem~\ref{thm:parity_stability}), since scaffold invariance is enforced by freezing and flow isolation is enforced by snapshotting.

\paragraph{Memory scaling.}
The total parameter storage for $T$ tasks is:
\begin{equation}
    \text{Parity-MTF:} \quad M_{\text{P-MTF}}(T) = P_S + T \cdot P_F,
\end{equation}
compared to standard MTF:
\begin{equation}
    \text{MTF:} \quad M_{\text{MTF}}(T) = T \cdot P.
\end{equation}
The memory savings ratio is:
\begin{equation}
    1 - \frac{M_{\text{P-MTF}}(T)}{M_{\text{MTF}}(T)} = 1 - \frac{P_F}{P} - \frac{P_S}{T \cdot P} \xrightarrow{T \to \infty} 1 - \frac{P_F}{P} = \frac{P_S}{P}.
    \label{eq:amortization}
\end{equation}
The scaffold cost $P_S$ amortizes over tasks: as $T$ grows, the per-task storage approaches $P_F$, and the savings ratio converges to the scaffold fraction $P_S / P$.

\paragraph{Layer assignment for Permuted-MNIST.}
A critical design consideration is which layers serve as scaffold versus flow.
For Permuted-MNIST, each task applies a different random permutation to the input pixel order.
Since the input distribution changes completely per task, the input-facing layer (Linear: $784 \to h$) \emph{must} be flow, a frozen input layer trained on permutation~0 produces incoherent features for all other permutations.
Middle layers that process abstract features can be shared.
We use a 4-layer MLP with layers 1, 2 as scaffold ($P_S \approx 33\%$ of total) and layers 0, 3 plus the classification head as flow ($P_F \approx 67\%$ of total).

\subsection{Experimental Setup}

\paragraph{Benchmark.}
Permuted-MNIST with $T \in \{10, 25, 50, 100\}$ sequential tasks.
Each task applies a distinct random permutation to the $784$ input pixels and defines 10-class digit classification under the permuted ordering.
We use $1{,}000$ training and $200$ test samples per task, trained for $5$ epochs with Adam ($\text{lr} = 10^{-3}$, batch size $64$). All experiments are conducted on Google Colab with access to A100 GPU; the total running time is around 30 minutes.

\paragraph{Methods compared.}
All methods except EWC use the same 4-layer MLP architecture ($\approx\!401$K params) for fair comparison:
\begin{enumerate}
    \item \textbf{EWC} ($\lambda = 400$): A single shared model with Fisher-weighted regularization, requiring known task boundaries. Skipped for $T > 25$ due to quadratic time scaling from Fisher accumulation.
    \item \textbf{MTF} (full experts): One complete 4-layer classifier per task, frozen after training. Zero forgetting by construction.
    \item \textbf{Parity-MTF}: Shared scaffold (layers 1, 2) frozen after task~0; per-task flow snapshots (layers 0, 3, head) re-initialized and trained independently. Zero forgetting by construction.
\end{enumerate}

\subsection{Experimental Results of Scaling Studies}

\paragraph{Accuracy and forgetting.}
Table~\ref{tab:scalability} reports the primary metrics.
Both MTF and Parity-MTF maintain near-constant accuracy ($\approx\! 85$--$88\%$) as $T$ scales from 10 to 100, with exactly zero forgetting by construction.
EWC degrades substantially, dropping from $\approx\!72\%$ at $T=10$ to $\approx\!44\%$ at $T=25$, consistent with the theoretical prediction that regularization-based methods suffer capacity saturation as Fisher penalties compound over tasks.
The accuracy gap between MTF and Parity-MTF ($\approx\!2$--$3\%$) reflects the cost of scaffold sharing: the scaffold trained on task~0 provides a fixed feature basis that may not be optimal for later permutations.
This gap remains bounded and does not widen with $T$, confirming that the scaffold captures genuinely reusable abstract structure.

\begin{table}[h]
\centering
\caption{Scalability results on Permuted-MNIST. Avg Acc: mean accuracy across all tasks after training on the final task. Forgetting: mean drop from peak per-task accuracy. Params: total stored parameters. Both MTF variants achieve zero forgetting; EWC degrades as tasks accumulate.}
\label{tab:scalability}
\vspace{0.1in}
\small
\begin{tabular}{clccrr}
\toprule
$T$ & Method & Avg Acc $\uparrow$ & Forget $\downarrow$ & Params & Time (s) \\
\midrule
\multirow{3}{*}{10}
& EWC ($\lambda\!=\!400$) & 0.718 & 0.107 & 401K & 8 \\
& MTF (full) & 0.875 & 0.000 & 4.0M & 2 \\
& Parity-MTF & 0.848 & 0.000 & 2.8M & 2 \\
\midrule
\multirow{3}{*}{25}
& EWC ($\lambda\!=\!400$) & 0.435 & 0.189 & 401K & 40 \\
& MTF (full) & 0.880 & 0.000 & 10.0M & 4 \\
& Parity-MTF & 0.852 & 0.000 & 6.9M & 5 \\
\midrule
\multirow{2}{*}{50}
& MTF (full) & 0.878 & 0.000 & 20.0M & 7 \\
& Parity-MTF & 0.856 & 0.000 & 13.6M & 9 \\
\midrule
\multirow{2}{*}{100}
& MTF (full) & 0.885 & 0.000 & 40.1M & 15 \\
& Parity-MTF & 0.860 & 0.000 & 27.1M & 18 \\
\bottomrule
\end{tabular}
\end{table}

\paragraph{Scaffold amortization.}
Fig.~\ref{fig:scaffold_amortization} illustrates the memory scaling behavior.
The total parameter cost of MTF scales as $T \cdot P$, while Parity-MTF scales as $P_S + T \cdot P_F$.
Since $P_F / P \approx 0.67$ for our architecture, the savings ratio converges to $\approx\!33\%$ as $T$ grows (Eq.~\ref{eq:amortization}).
At $T = 100$, Parity-MTF stores $\approx\!27$M parameters compared to MTF's $\approx\!40$M, a $32.5\%$ reduction with only $\approx\!2.5\%$ accuracy cost.
The per-task amortized cost (including the scaffold fraction $P_S / T$) drops from $\approx\!282$K at $T = 10$ toward the asymptote of $P_F \approx 269$K as $T \to \infty$.

\begin{figure}[h]
    \centering
    \includegraphics[width=\linewidth]{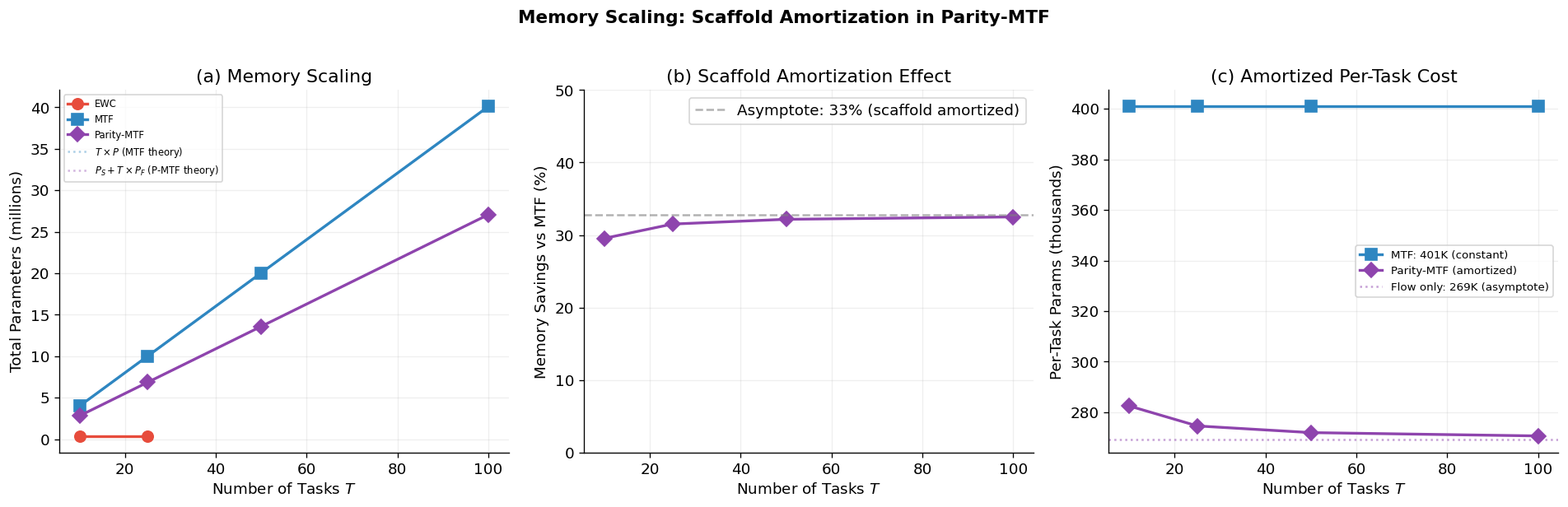}
    \caption{\textbf{Scaffold amortization in Parity-MTF.}
    \textbf{(a)} Total parameter storage scales linearly for both methods, but with different slopes: MTF at slope $P$ and Parity-MTF at slope $P_F < P$.
    The scaffold cost $P_S$ appears as a constant offset that becomes negligible relative to $T \cdot P_F$.
    \textbf{(b)} Memory savings converge to the scaffold fraction $P_S / P \approx 33\%$ (dashed line) as $T$ grows, confirming Eq.~\ref{eq:amortization}.
    \textbf{(c)} Amortized per-task cost: MTF is constant at $P$ per expert; Parity-MTF starts at $P_F + P_S / T$ and decreases toward $P_F$ (dotted line).}
    \label{fig:scaffold_amortization}
\end{figure}

\paragraph{Time scaling.}
Both MTF and Parity-MTF train each task independently, yielding $O(T)$ wall-clock scaling.
Parity-MTF trains fewer parameters per task (only $P_F$), but incurs a slight overhead from snapshot save/load operations, resulting in comparable training times.
EWC exhibits $O(T^2)$ scaling because the Fisher penalty at task $t$ sums over all $t-1$ previous parameter snapshots, making each gradient step increasingly expensive.
This quadratic scaling renders EWC impractical beyond $T \approx 25$ tasks in our setup.

\subsection{Analysis: Connection to the Urysohn Ladder}

The scalability results admit a precise interpretation through the paper's geometric framework.

\paragraph{Covering number perspective.}
Each Permuted-MNIST task defines a distinct context, so the Urysohn width is $w^* = T$.
By Theorem~\ref{thm:bounded_capacity}, the covering number of the quotient manifold is $N(\epsilon, \mathcal{M}_D) = w^*$, independent of the stream length within each context.
MTF realizes this bound exactly: each expert is one covering element.
Parity-MTF achieves the same covering number but with a more efficient parameterization: the scaffold provides the shared quotient structure (the topology of the quotient graph $G_Q$), while the per-task flow parameters encode the within-cell metric specific to each quotient cell.

\paragraph{Scaffold as quotient topology, flow as fiber metric.}
In the language of fiber bundles, the scaffold $\theta_S$ encodes the \emph{base space} (the shared structure that all tasks navigate), while the flow $\theta_F^{(t)}$ for task $t$ encodes the \emph{fiber} above the $t$-th quotient cell (the local metric adaptation needed for that specific permutation).
The scaffold cost $P_S$ is the fixed overhead of maintaining the quotient graph; the flow cost $T \cdot P_F$ is the total cost of parameterizing all fibers.
As $T \to \infty$, the scaffold amortizes and the per-task cost approaches the fiber cost $P_F$, consistent with the bounded capacity theorem's prediction that representational demand per quotient cell is $O(1)$.

\paragraph{The accuracy-memory Pareto frontier.}
The $\approx\!2.5\%$ accuracy gap between MTF and Parity-MTF represents the price of \emph{forward transfer} vs.\ \emph{independence}.
MTF experts are fully independent: each task gets its own input layer trained from scratch, with no constraint from other tasks.
Parity-MTF's scaffold is trained once on task~0 and reused for all subsequent tasks; if the scaffold's abstract features are not perfectly transferable to later permutations, a small accuracy penalty results.
However, this gap does not widen with $T$, confirming that the scaffold captures genuinely reusable structure, the \emph{content}, while the flow captures task-specific \emph{context}.

\paragraph{EWC as a monolithic counter-example.}
The degradation of EWC with increasing $T$ illustrates the geometric prediction: regularization-based approaches operate on a \emph{flat} parameter manifold and attempt to balance competing Fisher penalties.
As $T$ grows, the Fisher constraints from past tasks compound, progressively restricting the volume of parameter space available for new learning.
This is the capacity overflow predicted by Lemma~\ref{lem:linear_growth}: without metric contraction, the effective demand $N(\epsilon, \mathcal{M}_0) \propto T$ eventually exceeds the fixed capacity $d$ of the shared parameter set.
Parity-MTF avoids this failure mode entirely by implementing topological isolation (Theorem~\ref{thm:parity_stability}): scaffold invariance prevents cross-task interference, and per-task flow snapshots prevent capacity overflow.

\subsection{Implications for Path toward AGI}
There exist open-ended learning problems (width growing without bound, structural gap vanishing) for which no finite-resource learner can maintain vanishing average regret.
This is not a limitation of current algorithms; it is a mathematical impossibility akin to the halting problem in computation or G\"{o}del's incompleteness in formal systems.
Any claim of universal open-ended intelligence must therefore be qualified: it can only apply to the tractable subclass of open-ended problems whose Urysohn width remains bounded or grows sufficiently slowly relative to the learner's capacity.

This impossibility result does not preclude AGI but \emph{characterizes} it as following.
Biological general intelligence is precisely a system that operates within the tractable subclass while appearing general: the human brain handles an enormous but ultimately bounded range of cognitive contexts (languages, motor skills, social models, domain expertise) by maintaining bounded effective width through recursive quotienting.
Each new experience is either \emph{recognized} (routed to an existing cortical module, adding no structural complexity) or \emph{genuinely novel} (allocated to a new module via hippocampal pattern separation, increasing $K$ by one).
The brain's generality does not arise from solving arbitrary problems on a flat manifold; it arises from an architecture that \emph{autonomously manages its own topological complexity}, compressing the width of its experiential stream to remain within the bounded-capacity regime where the convergence guarantees of Theorem~\ref{thm:bounded_capacity} apply.

From this perspective, a plausible path toward AGI is not the indefinite scaling of monolithic parameter tensors (which attacks the funnel while ignoring the trap), but the engineering of three capabilities that the MTF framework identifies as jointly sufficient:
\begin{enumerate}
    \item \textbf{Autonomous width management.}
    The system must discover, track, and bound its own Urysohn width $w^*$ without external supervision.
    This requires a gradient-free topological indexer $\Sigma$ that detects context switches in an open-ended stream, recognizes recurring contexts, and allocates new capacity only when genuinely novel structure is encountered.
    The Urysohn Machine's E-D-C cycle is a proof of concept; a frontier implementation would require $\Sigma$ to operate on richer representations than raw-pixel centroids - e.g., on learned embeddings from a pre-trained foundation model serving as the metric slingshot $\varphi$.
    
    \item \textbf{Recursive condensation at scale.}
    The system must implement the Urysohn Ladder's quotient hierarchy: validated sub-trajectories are collapsed into reusable tokens (the condensation operator $\Psi: \mathcal{H}_{\mathrm{odd}} \to \mathcal{H}_{\mathrm{even}}$), converting each unit of search cost into a permanent unit of navigational structure.
    This is the mechanism by which effective width is kept bounded even as raw experiential complexity grows without bound: the quotient tower absorbs new contexts by folding them into the existing scaffold, rather than expanding the scaffold's dimensionality.
    Biologically, this is sleep-mediated consolidation; computationally, it is the ``write'' phase of MAI that converts NPSPACE search into P-time navigation.
    
    \item \textbf{A learnable metric slingshot with fiber-preserving stability.}
    The shared encoder $\varphi$ must be continually improvable without corrupting previously consolidated routing contexts.
    This is the narrowest open problem in the MTF framework and, we argue, the central unsolved problem for AGI: how to learn a universal representation that respects the topological structure of all previously learned contexts.
    A solution would likely combine pre-trained foundation models (which provide a strong initial $\varphi$) with fiber-preserving fine-tuning (which expands $\varphi$ into new representational directions without overwriting old ones).
    The biological analog is the slow structural plasticity of the neocortical connectome, which reorganizes over months to years while preserving the functional connectivity patterns that support existing skills.
\end{enumerate}

Under this formulation, AGI is not a system that solves all possible problems (which is provably impossible) but one that \emph{autonomously maintains bounded width over an open-ended experiential stream}, recognizing what it has seen, allocating for what is genuinely new, and compressing what it has learned into navigable structure.
The impossibility result tells us where the boundary is; the MTF architecture tells us how to operate productively at that boundary.
The gap between current systems and this vision is not primarily one of scale (more parameters, more data, more compute) but of \emph{architectural completeness}: current foundation models are powerful funnels that lack a trap, powerful navigators that lack a map-builder, powerful pattern-matchers that lack an autonomous indexer.
Closing this gap, equipping a frontier model with $\Sigma$, $\Psi$, and a fiber-preserving $\varphi$, is, in our view, the most concrete and theoretically grounded path toward artificial general intelligence.

\section{Declaration about the Usage of LLM}

%\paragraph{Declaration of LLM Usage.}
Large language models (Gemini3.1 Pro, Claude Opus 4.6 by Anthropic and ChatGPT 5.4 by OpenAI) were used as assistive tools during the preparation of this work in the following capacities:
(1)~drafting and iterating on experimental code, including Jupyter notebooks for benchmark implementations and baseline adaptations;
(2)~generating initial drafts of tables, and figure captions, which were subsequently reviewed, revised, and verified by the authors;
(3)~debugging code, checking numerical consistency between notebook outputs and manuscript tables, and suggesting structural organization for the experimental sections.
All theoretical contributions (definitions, theorems, proofs), experimental design decisions, scientific interpretations, and final manuscript content were conceived, validated, and approved by the authors.
The authors take full responsibility for the correctness and integrity of the published work.

%%%%%%%%%%%%%%%%%%%%%%%%%%%%%%%%%%%%%%%%%%%%%%%%%%%%%%%%%%%%
%\newpage
%\input{checklist.tex}

\end{document}